%% file: acl_latex.tex
\pdfoutput=1

\documentclass[dvipsnames, 11pt]{article}

\usepackage[preprint]{acl}

\usepackage{times}
\usepackage{latexsym}

\usepackage[T1]{fontenc}

\usepackage[utf8]{inputenc}

\usepackage{microtype}

\usepackage{inconsolata}

\usepackage{colortbl}
\usepackage{adjustbox}
\usepackage{diagbox}
\usepackage{amsmath,amssymb, amsthm}
\usepackage[inline]{enumitem}
\usepackage{array,multirow,graphicx,makecell}
\usepackage{booktabs}
\usepackage{tabularx}
\usepackage{xfrac}
\usepackage[skins]{tcolorbox}
\tcbuselibrary{breakable}
\usepackage{anyfontsize}
\usepackage{framed}
\usepackage{tabto}
\usepackage{tabulary}
\usepackage{anyfontsize}
\usepackage{xfrac}
\usepackage[skins]{tcolorbox}
\usepackage{algpseudocode}
\usepackage{algorithm}
\usepackage{listings}
\usepackage{anyfontsize}
\usepackage[toc,page,header]{appendix}
\usepackage{minitoc}

\newcounter{tcbcounter}

\input{math_commands.tex}

\newtheorem{theorem}{Theorem}

\newif\ifcomments
\commentstrue
\ifcomments
    \providecommand\roi[1]{\textcolor{purple}{[Roi: #1]}}    \providecommand\rotem[1]{\textcolor{blue}{[Rotem: #1]}}
    \providecommand\nitay[1]{\textcolor{magenta}{[Nitay: #1]}}
    \providecommand\todo[1]{\textcolor{red}{[TODO: #1]}}
\else
    \providecommand{\roi}[1]{}    \providecommand{\rotem}[1]{}
    \providecommand{\nitay}[1]{}
    \providecommand{\todo}[1]{}
\fi

\newcommand{\RMSE}{\texttt{RMSE}}
\newcommand{\ACC}{\texttt{ACC}}
\newcommand{\SIM}{\texttt{SIM}}

\newtcolorbox[auto counter, number within=section]{prompt}[3][]{%
  enhanced,
  breakable,
  colback=#2!5!white,
  colframe=#2!75!black,
  title=\textbf{Box \thetcbcounter: #3},
  fontupper=\footnotesize\fontfamily{cmr}\selectfont,
  #1
}

\newtcolorbox{quotebox}{
enhanced,
boxrule=0pt,frame hidden,
borderline west={4pt}{0pt}{CadetBlue!90!black},
colback=CadetBlue!20!white,
sharp corners
}

\title{The Alternative Annotator Test for LLM-as-a-Judge: \\ How to Statistically Justify Replacing Human Annotators with LLMs}

\author{Nitay Calderon$^T$ \And Roi Reichart$^T$ \\
        $^T$Faculty of Data and Decision Science, Technion \\ 
        $^H$Faculty of Computer and Information Science, University of Haifa \\ 
        \texttt{nitay@campus.technion.ac.il} \quad \texttt{roiri@technion.ac.il} \quad
        \texttt{rdror@is.haifa.ac.il} \\
        \And Rotem Dror$^H$}

\begin{document}

\maketitle

\doparttoc 
\faketableofcontents 

\begin{abstract}
The ``LLM-as-an-annotator'' and ``LLM-as-a-judge'' paradigms employ Large Language Models (LLMs) as annotators, judges, and evaluators in tasks traditionally performed by humans. LLM annotations are widely used, not only in NLP research but also in fields like medicine, psychology, and social science. Despite their role in shaping study results and insights, there is no standard or rigorous procedure to determine whether LLMs can replace human annotators. In this paper, we propose a novel statistical procedure, the Alternative Annotator Test (alt-test), that requires only a modest subset of annotated examples to justify using LLM annotations. Additionally, we introduce a versatile and interpretable measure for comparing LLM annotators and judges. To demonstrate our procedure, we curated a diverse collection of ten datasets, consisting of language and vision-language tasks, and conducted experiments with six LLMs and four prompting techniques. Our results show that LLMs can sometimes replace humans with closed-source LLMs (such as GPT-4o), outperforming the open-source LLMs we examine, and that prompting techniques yield judges of varying quality. We hope this study encourages more rigorous and reliable practices.
\footnote{Code for the procedure and datasets are available at: \url{https://github.com/nitaytech/AltTest}}
\end{abstract}

\input{sections/intro}
\input{sections/arr/related_arr}
\input{sections/method}

\input{sections/experiments}

\input{sections/results}
\input{sections/conclusion}
\input{sections/limitations}

\bibliography{custom}

\appendix

\renewcommand \thepart{}
\renewcommand \partname{}
\mtcsettitle{parttoc}{}
\addcontentsline{toc}{section}{Appendix} 
\part{Appendix} 
\parttoc 

\clearpage


\input{sections/arr/appendix_arr}

\end{document}

%% file: math_commands.tex
\usepackage{amsmath,amsfonts,bm}

\def\eqref#1{equation~\ref{#1}}

\def\1{\bm{1}}

\DeclareMathAlphabet{\mathsfit}{\encodingdefault}{\sfdefault}{m}{sl}
\SetMathAlphabet{\mathsfit}{bold}{\encodingdefault}{\sfdefault}{bx}{n}

\def\sH{{\mathbb{H}}}
\def\sI{{\mathbb{I}}}

\newcommand{\E}{\mathbb{E}}



%% file: sections/intro.tex
\section{Introduction}

The rise of Large Language Models (LLMs) has transformed the field of Natural Language Processing (NLP), bringing unprecedented capabilities in reasoning and generating human-like text \citep{KojimaGRMI22, achiam2023gpt, laskar-etal-2023-systematic, YangJTHFJZYH24}. Recently, a new trend has emerged where LLMs are employed as annotators and judges across various NLP applications \citep{li2024generation, TanLWBJBKL0024}.

One key advantage of LLM-as-an-annotator and LLM-as-a-judge\footnote{The term ``LLM-as-a-judge'' typically refers to LLMs evaluating outputs of other LLMs. It can be viewed as a special case of the broader ``LLM-as-an-annotator'' paradigm. However, since ``LLM-as-a-judge'' is more widely used, we adopt it throughout this work to refer more generally to any evaluation, annotation, or labeling of texts (or images) traditionally performed by humans, regardless of the input source.} paradigms is the scalability and speed of LLMs. They can quickly annotate large-scale datasets, reducing the time required for tasks traditionally performed by costly human annotators \cite{NasutionO24}. LLMs also avoid challenges inherent to human factors, such as fatigue and guideline misinterpretation \cite{uma2021learning, bartsch2023self}. In certain cases, they even outperform crowd-workers \citep{Gilardi2023chat, nahum2024llms}.

Indeed, LLMs-as-judges are extensively used in research, taking on a pivotal role once filled by humans. They are employed to annotate new datasets \citep{GatCFCSR24, TanLWBJBKL0024}, or refine existing ones \citep{nahum2024llms, Pavlovic2024Effectiveness}, and commonly serve as evaluators for benchmarking models and methods \citep{Ahmed2024Can, gu2024survey, li2024generation}.

LLMs' influence extends far beyond the NLP field.
They annotate papers for literature reviews \citep{calderon2024behalf, Joos2024Cutting} or extract findings from academic literature \citep{khraisha2024can, NaikKBDH24}. They are also utilized in cognitive sciences to simulate human subjects \citep{AherAK23, Shapira2024Can, trott2024large} and in social science, researchers leverage LLM annotations to uncover social and cultural insights \citep{Ventura2023Navigating, ZiemsHSCZY24}. 
Accordingly, LLMs directly shape the results, findings, and insights of studies and guide the direction of scientific inquiry, prioritization, and innovation. 

Despite the advantages of the LLM-as-a-judge paradigm, research shows that LLMs amplify biases, leading to unfair or inconsistent judgments \citep{Ashktorab2024Aligning, chen2024humans, Ye2024Justice} and that they may struggle with tasks that require deep contextual understanding or domain-specific expertise \citep{ravid2023140, Szymanski2024Limitations}. These weaknesses highlight the need for rigorous evaluation and transparency when relying on LLM annotations in research. 

Yet, many studies employing LLM annotations do not explicitly measure the alignment between LLMs and humans, and those that do typically use traditional measures such as accuracy (\% agreements), F1 score, Inter-Annotator-Agreement (IAA) kappas, and correlation \citep{li2024llms}, which have limitations.
To start, IAA measures assess agreement among a group of annotators, while we aim to compare the LLM to the group. Other measures frequently rely on majority vote labels, overlooking important nuances that individuals introduce. Moreover, there are no established criteria for making a definitive yes/no decision on whether an LLM can replace humans (e.g., \textit{``is an F1 score of 0.6 sufficient?''}). This decision demands statistical rigor, which often lacks in the way researchers apply traditional measures.
Finally, they can only evaluate whether an LLM \textit{matches} human performance (i.e., is bounded by it) but cannot determine whether it provides a \textit{better} alternative.

We argue that to justify using an LLM instead of human annotators, researchers should demonstrate that \textit{the LLM offers a better alternative to recruiting human annotators.} In other words, when factoring in the cost-benefit and efficiency advantages of LLM annotations, they should be as good as or better than human annotations. In this paper, we propose a statistical procedure to verify this claim, which we call \textit{the Alternative Annotator Test}, or simply \textit{alt-test}. This procedure is simple and requires minimal effort to apply; it involves comparing the LLM to a small group of human annotators (at least three) on a modest subset of examples (between 50 and 100). Our procedure is described in \S\ref{sec:method} and illustrated in Figure~\ref{fig:intro_fig}. Once applied, researchers can confidently rely on the LLM's annotations for their work. 

In addition, we define a measure for comparing LLM judges called the \textit{Average Advantage Probability}. This measure is naturally derived from our statistical procedure and represents the probability that the LLM annotations are as good as or better (e.g., by being closer to the majority) than those of a randomly chosen human annotator. It possesses desirable properties that traditional measures lack while maintaining a high correlation with them. It is versatile, supports different types of annotations, and is highly interpretable.

We exemplify the application of our procedure with six LLMs and four prompting techniques. To this end, we curate a diverse collection of ten datasets, each with instances annotated by multiple annotators. Our datasets vary in size, annotation types (discrete, continuous, and free-text), number of annotators (3 to 13), and levels of annotator expertise (crowd-workers, skilled annotators, and experts). They encompass a wide range of language tasks, including two vision-language tasks.

Our results indicate that in many cases, LLMs can serve as an alternative to human annotators. Specifically, on nine datasets, at least one LLM, with some prompting technique, successfully passed the alt-test. We found that closed-source LLMs (such as GPT-4o and Gemini-1.5) consistently outperform open-source models we examined (like Mistral-v3 and Llama-3.1), and that in-context learning generally improves LLM performance, while chain-of-thought and ensemble methods do not yield similar benefits. 

Finally, in Appendix \ref{sec:advanced}, we propose modifications to our procedure to address advanced scenarios: handling imbalanced labels (\S\ref{sub:imbalanced}), benchmarking against a single expert (\S\ref{sub:single_expert}), incorporating annotator quality scores (\S\ref{sub:quality}), and respecting minority opinions in subjective tasks (\S\ref{sub:subjective}).

Our contributions are as follows: (1) We propose a statistical procedure, the alt-test, to justify replacing human annotators with LLMs; (2) We introduce a versatile and interpretable measure, the average advantage probability, for comparing LLM judges; (3) We curate a diverse collection of ten datasets and analyze six LLMs and four prompting techniques, demonstrating that LLMs can sometimes replace humans; (4) We develop a theorem regarding the optimal LLM-as-a-judge (\S\ref{sub:comparing}, \S\ref{sub:theorem}).

We encourage researchers to adopt our procedure and hope this study paves the way for rigorous scientific practices in NLP and beyond.

%% file: sections/arr/related_arr.tex
\section{Previous Work}

Research on LLMs as annotators and judges is a rapidly growing field \citep{chiang2023vicuna, zheng2024judging}, resulting in numerous surveys \citep{gu2024survey, li2024generation, TanLWBJBKL0024, Pavlovic2024Effectiveness}. Most studies focus on enhancing LLM performance, either by parameter tuning \citep{GekhmanHAES23, YueWCZS023, Zhu2023JudgeLM, JiangLZHLC24, KimS0JLLYSKTS24} or prompting strategies \citep{BaiY0LHWYZXLZLH23, Moniri2024debates, Song2024Many}. For instance, \citet{dong2024can} investigated personalized LLM judges, \citet{verga2024replacing} proposed using a panel of diverse LLMs, and \citet{chen2024mllm} extended LLM-as-a-judge to multimodal tasks.

Many statistical works propose corrections to estimations that are built with LLM annotations \citep{Angelopoulos2023PPI, EgamiHSW23, Angelopoulos2024PPI, Chatzi2024Prediction, Gligoric2024Unconfident, ludwig2024large}. Conversely, the question we address is how to justify replacing human annotators with LLMs, ensuring researchers can confidently apply LLMs for model evaluation or data annotation.

While existing works do not directly address how to justify human replacement, many have explored how well LLMs align with human annotators \citep{chiang2023can, Ahmed2024Can, Bavaresco2024LLMs, Chen0PLKP24, gera2024justrank, Lambert2024RewardBench, nahum2024llms, NasutionO24, Tan2024Judebench, trott2024large}, often focusing on specific LLM limitations or biases \citep{Wu2023Style, Ashktorab2024Aligning, jung2024trust, chen2024humans, WangLCCZLCKLLS24, xu-etal-2024-pride}. 
These studies rely on traditional measures such as accuracy, F1 score, correlation, or metrics that quantify bias. In contrast, we propose a statistical procedure to determine whether an LLM can be used, providing a clear yes/no answer. Additionally, we introduce an interpretable and versatile measure for comparing LLM judges.

%% file: sections/method.tex
\input{figures/intro_diag_latex}

\vspace{-1em}
\section{Method}
\label{sec:method}

We propose using an LLM-as-a-judge instead of human annotators when it offers a comparable alternative to recruiting an annotator. By comparing the predictions of the LLM to those of humans, we can evaluate which more closely emulates the gold label distribution. Gold labels represent the ``true'' or ground truth annotations and are typically determined through rigorous processes, such as consensus among experts or extensive quality control. Consequently, since experts are expensive and often inaccessible, we assume gold labels are unavailable. Hence, a common approach is to approximate them using the collective responses of multiple annotators. This is the exact setup we use in this paper: a modest subset of randomly sampled examples, each annotated by multiple annotators.\footnote{In \S\ref{sub:humans}, we discuss the number of annotators, their profiles, and levels of expertise to ensure reliable outcomes.}

Accordingly, a key consideration in our method is that the perspective of every annotator is valued. Specifically, our leave-one-out approach excludes one annotator at a time and evaluates how well the LLM's annotations align with those of the remaining annotators. Similarly, we evaluate the alignment of the excluded annotator with the remaining annotators. We then compare the LLM and the excluded annotator, justifying the use of the LLM-as-a-judge if \textit{the LLM aligns more closely with the collective distribution than an individual does}. The procedure is illustrated in Figure~\ref{fig:intro_fig}.

\paragraph{Notations and Definitions}
For a dataset of $n$ instances $\{x_1,\ldots,x_n\}$ and $m$ human annotators $\{h_1,\ldots,h_m\}$, we denote the annotation of the $j$th annotator for instance $x_i$ as $h_j(x_i)$. 
The annotation predicted by the LLM is denoted as $f(x_i)$.
In addition, $[-j]$ represents the set of indices from $1$ to $m$ excluding the $j$th index, i.e., $[-j]=\{1, \ldots, j-1, j+1, \ldots, m\}$.
The set of indices of the instances annotated by $h_j$ is denoted as $\sI_j$. Similarly, $\sH_i$ is the set of indices of human annotators that annotated $x_i$. For example, assume we have three instances and four annotators. $\sI_2=\{2,3\}$ means that the second annotator, $h_2$, annotated instances $x_2$ and $x_3$, and $\sH_1=\{1,3,4\}$ means that the first instance, $x_1$, was annotated by the first, third, and fourth annotators, $h_1, h_3, h_4$. 

\subsection{Computing the Instance Alignment Score}

We start by examining the removal of each human annotator $h_j$ in turn and compute a score that measures the alignment between the annotations of the $[-j]$ human annotators and the annotation of the LLM for instance $x_i$. We use $S(f, x_i, j)$ to denote the \textit{alignment scoring function} between $f(x_i)$ and the annotations of $\sH_i[-j]$.
For example, $S$ could be $\RMSE$ (root mean squared error) in regression tasks (continuous numerical labels) or $\ACC$ (accuracy) in classification tasks (categorical or rank labels). 

In generation tasks (e.g., machine translation), $S$ can be computed using a relevant evaluation metric (denoted as $\texttt{sim}$) that typically measures the similarity between the LLM-generated output and the human-generated output.
For convenience, we assume that higher values of $S$ indicate a better alignment between an LLM and the human annotators; thus, we use negative $\RMSE$. Below, we formally define the mentioned variants of $S$:
\resizebox{\columnwidth}{!}{%
\begin{minipage}{\columnwidth}
\begin{align*}
    -\RMSE(f, x_i, j) &= -\sqrt{\frac{1}{|\sH_i|-1}\sum_{k \in \sH_i[-j]}(f(x_i) - h_k(x_i))^2} \\
    \ACC(f, x_i, j) &= \frac{1}{|\sH_i|-1} \sum_{k \in \sH_i[-j]} \1\{f(x_i) = h_k(x_i)\} \\
    \SIM(f, x_i, j) &= \frac{1}{|\sH_i|-1} \sum_{k \in \sH_i[-j]} \texttt{sim}(f(x_i), h_k(x_i))
\end{align*}
\vspace{0.1em}
\end{minipage}}
Note that $-\RMSE(h_j, x_i, j)$,  $\ACC(h_j, x_i, j)$, and $\SIM(h_j, x_i, j)$ represent score differences between $h_j$ and the other annotators. Consequently, we are interested in comparing $S(f, x_i, j)$ to $S(h_j, x_i, j)$.

\subsection{Estimating the Advantage Probabilities}
After computing the alignment score for each instance, we estimate the likelihood that the LLM achieves a comparable alignment with the annotators to that of the excluded annotator.
The estimator will be constructed by calculating the percentage of instances for which the score of the LLM, $S(f, x_i, j)$, was higher or equal to the score of the $j$th excluded human annotator, $S(h_j, x_i, j)$. 
We represent this event (for $x_i$) using the indicator:
\begin{align*}
W_{i,j}^f = 
\begin{cases}
1, & \text{if } S(f, x_i, j) \ge S(h_j, x_i, j) \\
0, & \text{otherwise }
\end{cases}
\end{align*}
Similarly, we define the indicator $W_{i, j}^h$ by reversing the inequality (to $\le$) in the definition above, representing that the annotation of $h_j$ for $x_i$ is comparable to that of the LLM. 

The expectation of $W_{i,j}^f$ represents the probability that the LLM annotations are as good as or better than those of $h_j$. We estimate this probability by averaging $W_{i,j}^f$ values across all instances:
\begin{align*}
    \rho_j^f = \hat{\mathbb{P}}(\text{LLM}\succeq h_j) = \hat{\E}[W_{i,j}^f] = \frac{1}{|\sI_j|}\sum_{i \in \sI_j}{W_{i, j}^f}    
\end{align*}
We denote this estimation of the \textit{advantage over $h_j$ probability} as $\rho_j^f$. Similarly, $\rho_j^h$ estimates the probability that $h_j$ holds an advantage over the LLM, calculated by averaging the values of $W_{i, j}^h$. The set $\{(\rho_j^f, \rho_j^h)\}_{j=1}^m$ is used in our statistical procedure.


\subsection{Should the LLM Replace Annotators?}

Using an LLM instead of a human annotator is justified if the LLM offers a reliable alternative to hiring an annotator. To formalize this, if $\rho_j^f$ is \textbf{significantly} larger than $\rho_j^h$ it indicates that employing the LLM instead of $h_j$ is a \textit{justified evidence-based decision}. Notice, however, that employing an LLM is a cheaper and less labor-intensive alternative. Therefore, we introduce $\varepsilon$,\footnote{In \S\ref{sub:threshold} we explore how different $\varepsilon$ values impact our procedure and recommend suitable ones for researchers.} a \textit{cost-benefit hyperparameter} which penalizes $\rho_j^h$ to reflect the higher cost and effort associated with human annotation.

We define the following set of hypothesis testing problems to test if the LLMs' relative advantage probability is significantly larger than that of $h_j$:
\begin{align*}
    \mathbf{H_{0j}:} \rho_j^f \le \rho_j^h - \varepsilon \quad\text{vs.}\quad
    \mathbf{H_{1j}:} \rho_j^f > \rho_j^h - \varepsilon 
\end{align*}
The appropriate statistical test for this hypothesis problem is a paired $t$-test \cite{dror2018hitchhiker}, which examines the difference between the $i$th indicators: $d_{i,j}=W_{i,j}^h-W_{i,j}^f$. The null hypothesis asserts that $\bar{d}_{j}=\rho_j^h-\rho_j^f$ is greater than or equal to $\varepsilon$, while the alternative hypothesis posits that it is smaller.

The test statistic $t_j$ is defined as:
\begin{align*}
    t_j = \frac{\bar{d}_j-\varepsilon}{s_{j} / \sqrt{n}} \quad s_{j} = \sqrt{\frac{\sum_{i=1}^n \left(d_{i,j} - \bar{d}_{j}\right)^2}{n-1}}
\end{align*}
The p-value can be calculated using a student's $t$-distribution table. When $n<30$, the normality assumption may not hold, and a non-parametric test (e.g., Wilcoxon signed-rank) should be used. If the p-value $<\alpha$ (typically $\alpha = 0.05$),
we reject the null hypothesis, concluding that \textit{the LLM holds a statistically significant advantage over $h_j$ when considering the cost-benefit tradeoff.}

So far, we discussed the advantage of LLMs over a single human annotator. To generalize our conclusion to any annotator, we measure the percentage of annotators that the LLM ``wins'', i.e., the proportion of rejected null hypotheses. We denote this \textit{winning rate (WR)} by $\omega$, formally:
\resizebox{\columnwidth}{!}{%
\begin{minipage}{\columnwidth}
\begin{align*}
    \omega = \frac{1}{m}\sum_{j=1}^{m}{\1\{H_{0j} \text{ is rejected}\}}
\end{align*}
\end{minipage}}
where $\1\{H_{0j} \text{ is rejected}\}$ is an indicator that receive one if the null hypothesis is rejected and zero, otherwise. 
If $\omega \ge 0.5$,\footnote{This is a hyperparameter. It is set to 0.5 to establish that it is \textit{more likely} that the LLM holds an advantage over humans. Stricter thresholds can be used in high-stakes domains.} then the LLM wins the majority of human annotators, hence \textit{we assert that it can replace human annotators.}

\paragraph{Multiple Comparison Correction} 
Simply counting the number of rejected null hypotheses is problematic due to the accumulation of Type-I errors when performing multiple hypothesis tests, particularly when the hypotheses are dependent \cite{dror2017replicability}. In our case, the dependency arises because the score of $h_j$ relies on the annotations of the remaining $[-j]$ annotators (see how $S$ is defined). The standard practice to address this issue is a multiple comparison correction.

We suggest using a procedure that controls the false discovery rate (FDR), which is the expected proportion of false positives (incorrect rejections of null hypotheses) among all rejected hypotheses in a multiple-hypothesis testing scenario. In other words, the FDR-controlling procedure ensures that the observed WR $\omega$ is reliable and does not overestimate the true percentage of wins due to accumulated false rejections or dependence between hypotheses. We recommend using the Benjamini-Yekutieli (BY) procedure (\citet{benjamini2001control}, see Algorithm~\ref{alg:by} in the Appendix) to control the FDR, as it is specifically suited for scenarios where the null hypotheses are dependent.
In our experiments, we use the standard target FDR level of $q=0.05$ (i.e., in expectation, at most 5\% of the rejections will be false rejections).

\paragraph{Summary: the Alt-Test} As illustrated in Figure~\ref{fig:intro_fig}, the alt-test involves the following steps: First, we compute the set of probabilities $\{(\rho_j^f, \rho_j^h)\}_{j=1}^m$, where each $\rho_j$ represents the advantage of the LLM over $h_j$ and vice versa. Next, we conduct $m$ one-sample proportion t-tests for the difference $\rho_j^h-\rho_j^f$ against $\varepsilon$, resulting in a corresponding set of $m$ p-values. We then apply the BY procedure to these p-values, which identifies the set of rejected null hypotheses. Finally, we compute the winning rate (the proportion of rejected hypotheses) and if $\omega \ge 0.5$, we can statistically justify using LLM annotations.

\subsection{How to Compare LLM Judges?}
\label{sub:comparing}


In many scenarios, we wish to compare different LLM judges. 
While it is possible to compare LLMs by their winning rate ($\omega$), we argue this is suboptimal. First, $\omega$ does not account for the magnitude of the wins. For example, $\rho_j^f = 0.9$ and $\rho_j^f = 0.6$ contribute equally to $\omega$ if their respective null hypotheses are rejected. Second, $\omega$ depends on the value of $\varepsilon$, and third, the range of its possible values depends on the number of human annotators, making it a coarse measure. For instance, with only three annotators, $\omega$ value is limited to 0, \sfrac{1}{3}, \sfrac{2}{3}, 1.

\input{tables/iaa_main}

Therefore, for comparing LLM judges, we propose the \textit{Average Advantage Probability (AP)}:
\begin{align*}
    \rho = \frac{1}{m}\sum_{j=1}^{m}{\rho_j^f}
\end{align*}
We argue that $\rho$ is a good measure for comparing LLM judges due to its desirable properties. Unlike $\omega$, $\rho$ spans a denser range of values and accounts for the magnitude of $\rho_j^f$s. Furthermore, it is more interpretable than traditional measures like F1, Cohen's $\kappa$, or correlation --- it directly represents the probability that the LLM annotations are as good as or better than those of a randomly chosen annotator. This intuitive interpretation makes it accessible and meaningful for decision-makers. Finally, $\rho$ can be applied consistently across different types of annotation tasks (discrete, continues, and free-text), providing a unified evaluation framework that 
eliminates the need to switch between measures.

\paragraph{The Optimal LLM-as-a-Judge} 
We now turn to the question of what constitutes the optimal LLM-as-a-judge. We define it as an LLM that achieves an advantage probability of $\rho=1$ (since $\omega$ depends on $n$ and $\varepsilon$, we do not include it in the theorem). The optimal LLM-as-a-judge naturally depends on the choice of the scoring function, $S(f, x_i, j)$. The theorem below addresses two functions: $\ACC$ (for discrete tasks) and $-\RMSE$ (for continuous tasks). See Appendix~\ref{sub:theorem} for more details and the proof.

\begin{theorem}[Optimal LLM-as-a-Judge]
For a given dataset, let $S(f, x_i, j)$ be the alignment scoring function. The optimal LLM-as-a-judge, denoted as $f^*(x_i)$, is defined as follows:
\begin{itemize}
    \item If $S=\ACC$, 
    then $f^*(x_i)=MV(x_i)$, predicting the majority vote of the annotators for $x_i$.
    \item If $S = -\RMSE$, 
    then $f^*(x_i)=\frac{\sum_{k\in\sH_i}{h_k(x_i)}}{|\sH_i|}$, predicting the mean annotation for $x_i$.
\end{itemize}
In both cases, the optimal LLM-as-a-judge achieves an advantage probability of $\rho=1$.
\end{theorem}

%% file: figures/intro_diag_latex.tex
\begin{figure}[t]
    \centering
    \includegraphics[width=0.475\textwidth]{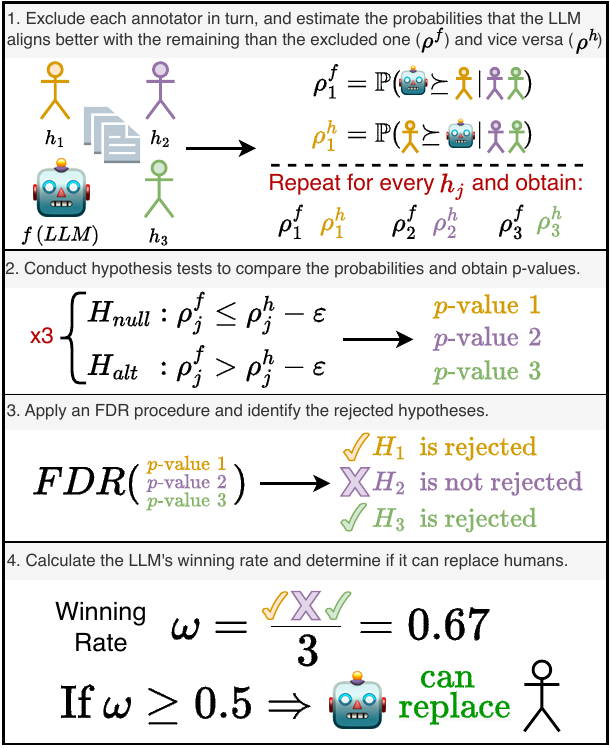}
    \caption{\textbf{An Illustration of the Alt-Test:} Given instances annotated by human annotators, we first exclude each annotator in turn to estimate the probabilities that the LLM better represents the remaining annotators and that the excluded annotator better represents them. We then test whether the LLM probability exceeds the annotator probability (considering a cost-benefit penalty $\varepsilon$), and apply a False Discovery Rate (FDR) controlling procedure. Then, we calculate the winning rate, $\omega$, as the proportion of rejected hypotheses. If $\omega \geq 0.5$, we conclude that the LLM is more likely to hold an advantage over human annotators, which justifies using it.}
    \label{fig:intro_fig}
\vspace{-0.8em}
\end{figure}

%% file: tables/iaa_main.tex
\begin{table*}[t]
\centering
\large
\begin{adjustbox}{width=\textwidth}
\begin{tabular}{l|rrrrr|cc|l}
\toprule
\multicolumn{9}{c}{\textbf{Discrete Annotation Tasks}} \\
\midrule
Dataset & $m$ & $n$ & Cats & I.p.A & A.p.I & Agree & Fleiss's $\kappa$ & Task Description \\
\midrule
WAX & 8 C & 246 & 16 & 172 & 5.61 & 0.33 & 0.26 & Identify the type of relationship between two associated words. \\
LGBTeen & 4 E & 880 & 5 & 640 & 2.91 & 0.69 & 0.53 & Assess the emotional support provided by LLMs to queer youth. \\
MT-Bench & 3 E & 120 & 3 & 82 & 2.05 & 0.66 & 0.49 & Compare two conversations between a user and different LLMs. \\
Framing & 4 S & 2552 & 3 & 1914 & 3.00 & 0.79 & 0.57 & Annotate climate articles with frame-related yes/no questions. \\
CEBaB-A & 10 C & 1008 & 3 & 403 & 4.00 & 0.86 & 0.74 & Determine the sentiment for four aspects of restaurant reviews.  \\
\bottomrule
\toprule
\multicolumn{9}{c}{\textbf{Continuous Annotation Tasks}} \\
\midrule
Dataset & Anns & Items & Scale & I.p.A & A.p.I & MAE & Pearson & Task Description \\
\midrule
SummEval & 3 E & 6400 & 1--5 & 6400 & 3.00 & 0.51 & 0.74 & Rate model-generated summaries on four aspects. \\
10k Prompts & 13 S & 1698 & 1--5 & 296 & 2.26 & 0.84 & 0.41 & Rate the quality of synthetic and human-written prompts. \\
CEBaB-S & 10 C & 711 & 1--5 & 219 & 3.08 & 0.67 & 0.67 & Identify the star rating (1-5) given in restaurant reviews. \\
\raisebox{-0.2em}{\includegraphics[height=1em]{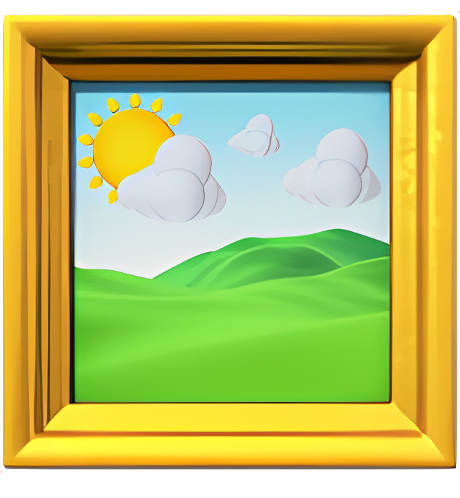}} Lesion & 6 S & 500 & 1--6 & 497 & 5.96 & 0.44 & 0.77 & Score five melanoma-related features based on lesion images. \\
\bottomrule
\toprule
\multicolumn{9}{c}{\textbf{Free-Text Annotation Tasks}} \\
\midrule
Dataset & Anns & Items & -- & I.p.A & A.p.I & \multicolumn{2}{c|}{Avg. Similarity} & Task Description \\
\midrule
\raisebox{-0.2em}{\includegraphics[height=1em]{figures/framed_picture.png}} KiloGram & 50 C & 993 & -- & 144 & 7.27 & \multicolumn{2}{c|}{0.28} & Generate free-text descriptions of tangram images. \\
\bottomrule
\end{tabular}
\end{adjustbox}
\caption{\textbf{Details of the Ten Datasets:} The number of human annotators ($m$), data instances ($n$), and categories (Cats). The letter in the `$m$' column indicates the type of annotators: Experts (E), Skilled (S), or Crowd-workers (C).
I.p.A and A.p.I denote the average numbers of items per annotator and annotators per item, respectively. For discrete tasks, we compute the proportion of pairwise agreements between human annotators (Agree) and Fleiss's $\kappa$. For continuous tasks, we compute the mean absolute error between annotators (MAE) and the average Pearson correlation. For the text generation task, we compute the average embedding cosine similarity (see Table~\ref{tab:results_kilogram}).}
\label{tab:iaa_main}
\vspace{-0.8em}
\end{table*}

%% file: sections/experiments.tex
\input{tables/results_main}

\section{Experimental Setup}

\subsection{Datasets}

We conducted experiments on ten diverse datasets, varying in size, number of human annotators, and types of annotators (crowd-workers, skilled annotators, or experts). Table~\ref{tab:iaa_main} provides information about these datasets, including inter-annotator agreement measures. We comprehensively review each of the ten datasets in Appendix~\ref{app:dataset_overview}.

The datasets span a broad range of tasks, including traditional NLP tasks like sentiment analysis, word-relation labeling, and summarization evaluation, as well as modern LLM-related tasks like conversation comparison, prompt quality assessment, and emotional support evaluation. Moreover, two datasets address vision-language tasks: skin lesion examination and abstract visual reasoning.

The selection of the datasets followed three principles: (1) covering diverse annotation types, including discrete, continuous, and free-text; (2) ensuring annotators have identifiers; and (3) requiring each item be annotated by multiple annotators.

\subsection{LLMs}
The six models that were used as candidate LLM annotators for our experiments are \textit{Gemini-1.5-Flash and Pro}\footnote{\url{https://deepmind.google/technologies/gemini/}} by Google DeepMind, \textit{GPT-4o and GPT-4o-mini}\footnote{\url{https://openai.com/index/hello-gpt-4o/}} by Open AI, \textit{Llama-3.1-7B-Instruct}\footnote{\url{https://www.llama.com/docs/model-cards-and-prompt-formats/llama3_1/}} by Meta AI, and \textit{Mistral-7B-Instruct-v0.3}\footnote{\url{https://writingmate.ai/blog/mistral-7b-v03-guide-and-details}} by Mistral AI. Llama-3.1 and Mistral-v3 do not have results on Lesion and KiloGram datasets because they are not able to process images.
The prompts used in our experiments are detailed in Appendix~\ref{sec:app_prompts}, and, where applicable, adhere to the annotation guidelines outlined in the papers describing the dataset.

In addition to the basic \textit{Zero-shot} strategy, we experimented with three advanced LLM-as-a-judge strategies \citep{li2024generation}: \textit{Few-shot} (also known as In-Context Learning), where the prompt includes four randomly sampled demonstrations (an input paired with its majority vote label); \textit{Chain-of-Thought (CoT)}, where the prompt instructs the LLM to reason step-by-step and provide an explanation before making a prediction; and \textit{Ensemble}, where the final prediction is determined by the majority label across an ensemble of LLMs and different prompting strategies \citep{nahum2024llms}.

%% file: tables/results_main.tex
\begin{table*}[t]
\centering
\large
\begin{adjustbox}{width=0.98\textwidth}
\begin{tabular}{l|ccc|ccc|ccc|ccc|ccc}
\toprule
\multicolumn{16}{c}{\textbf{Discrete Annotation Tasks}} \\
\midrule
& \multicolumn{3}{c|}{\textbf{WAX} ($\varepsilon=0.1$)} & \multicolumn{3}{c|}{\textbf{LGBTeen} ($\varepsilon=0.2$)} & \multicolumn{3}{c|}{\textbf{MT-Bench} ($\varepsilon=0.2$)} & \multicolumn{3}{c|}{\textbf{Framing} ($\varepsilon=0.15$)} & \multicolumn{3}{c}{\textbf{CEBaB-A} ($\varepsilon=0.1$)} \\
\midrule
& \underline{Acc} & \underline{WR $\omega$} & \underline{AP $\rho$} & \underline{Acc} & \underline{WR $\omega$} & \underline{AP $\rho$} & \underline{Acc} & \underline{WR $\omega$} & \underline{AP $\rho$} & \underline{Acc} & \underline{WR $\omega$} & \underline{AP $\rho$} & \underline{Acc} & \underline{WR $\omega$} & \underline{AP $\rho$} \\
Gemini-Flash & 0.38 &                    0.38 &          0.69 &          0.54 &                     0.25 &          0.71 &         0.62 &         0.0 &          0.72 &         0.69 & \cellcolor{green!30}1.0 &          0.83 &               0.88 & \cellcolor{green!30}0.7 &             0.91 \\
  Gemini-Pro & 0.39 & \cellcolor{green!30}0.5 & \textbf{0.74} &          0.47 &                      0.0 &          0.67 &         0.62 &         0.0 &          0.76 &         0.79 & \cellcolor{green!30}1.0 &          0.91 &               0.91 & \cellcolor{green!30}0.9 &    \textbf{0.94} \\
      GPT-4o & 0.38 & \cellcolor{green!30}0.5 &          0.73 &          0.63 & \cellcolor{green!30}0.75 & \textbf{0.77} &         0.68 &         0.0 & \textbf{0.77} &         0.80 & \cellcolor{green!30}1.0 & \textbf{0.92} &               0.90 & \cellcolor{green!30}0.9 &             0.93 \\
 GPT-4o-mini & 0.24 &                     0.0 &          0.59 &          0.59 & \cellcolor{green!30}0.75 &          0.76 &         0.60 &         0.0 &          0.74 &         0.74 & \cellcolor{green!30}1.0 &          0.87 &               0.86 & \cellcolor{green!30}0.5 &              0.90 \\
   Llama-3.1 & 0.24 &                     0.0 &          0.57 &          0.54 &                      0.0 &          0.72 &         0.54 &         0.0 &          0.69 &         0.66 & \cellcolor{green!30}0.5 &           0.80 &               0.87 & \cellcolor{green!30}0.6 &             0.89 \\
  Mistral-v3 & 0.17 &                     0.0 &           0.50 &          0.58 &                     0.25 &          0.75 &         0.52 &         0.0 &          0.68 &         0.66 &                    0.25 &           0.80 &               0.78 &                     0.1 &             0.81 \\
\bottomrule
\toprule
\multicolumn{16}{c}{\textbf{Continuous and Textual Annotation Tasks}} \\
\midrule
& \multicolumn{3}{c|}{\textbf{SummEval} ($\varepsilon=0.2$)} & \multicolumn{3}{c|}{\textbf{10K Prompts} ($\varepsilon=0.15$)} & \multicolumn{3}{c|}{\textbf{CEBaB-S} ($\varepsilon=0.1$)} & \multicolumn{3}{c|}{\textbf{Lesion} ($\varepsilon=0.15$)} & \multicolumn{3}{c}{\textbf{KiloGram} ($\varepsilon=0.1$)} \\
\midrule
& \underline{Pears} & \underline{WR $\omega$} & \underline{AP $\rho$} & \underline{Pears} & \underline{WR $\omega$} & \underline{AP $\rho$} & \underline{Pears} & \underline{WR $\omega$} & \underline{AP $\rho$} & \underline{Pears} & \underline{WR $\omega$} & \underline{AP $\rho$} & \underline{Sim} & \underline{WR $\omega$} & \underline{AP $\rho$} \\
Gemini-Flash &  0.51 & 0.0 &          0.46 &              0.44 &                     0.31 &           0.67 &              0.75 & \cellcolor{green!30}0.6 &           0.82 &         0.70 &                     0.17 &          0.71 &             0.79 & \cellcolor{green!30}0.66 &          \textbf{0.61} \\
  Gemini-Pro &  0.47 & 0.0 &          0.44 &              0.33 &                     0.08 &           0.63 &              0.78 & \cellcolor{green!30}0.8 &           0.87 &         0.73 &  \cellcolor{green!30}1.0 & \textbf{0.81} &             0.77 &                     0.08 &          0.43 \\
      GPT-4o &  0.54 & 0.0 &          0.48 &              0.47 & \cellcolor{green!30}0.69 &           0.76 &              0.80 & \cellcolor{green!30}0.9 &   \textbf{0.90} &         0.67 &                      0.0 &          0.62 &             0.78 &                      0.2 & 0.53 \\
 GPT-4o-mini &  0.50 & 0.0 &          0.54 &              0.46 & \cellcolor{green!30}0.92 &   \textbf{0.80} &              0.79 & \cellcolor{green!30}0.9 &           0.89 &         0.72 & \cellcolor{green!30}0.67 &          0.73 &             0.78 &                     0.16 &          0.49 \\
   Llama-3.1 &  0.36 & 0.0 &          0.58 &              0.23 &                     0.15 &           0.67 &              0.78 & \cellcolor{green!30}0.6 &           0.85 &          -- &                      -- &           -- &              -- &                      -- &           -- \\
  Mistral-v3 &  0.12 & 0.0 & \textbf{0.62} &              0.28 &                     0.15 &           0.67 &              0.76 & \cellcolor{green!30}0.5 &           0.83 &          -- &                      -- &           -- &              -- &                      -- &           -- \\
\bottomrule
\end{tabular}
\end{adjustbox}
\caption{\textbf{Main Results (zero-shot) — Full Datasets:} For all tasks, we report a traditional LLM-human alignment measure, such as accuracy with the majority vote (Acc) for discrete tasks, Pearson's correlation (Pears) for continuous tasks, and average similarity (Sim) for textual tasks. Additionally, we present our proposed measures: the winning rate (WR $\omega$, the $\varepsilon$ value is stated next to the dataset name) and the average advantage probability (AP $\rho$). Bold values indicate the best-performing LLM according to $\rho$, while a light green background highlights $\omega \ge 0.5$.}
\label{tab:results_main}
\vspace{-0.8em}
\end{table*}

%% file: sections/results.tex
\section{Results}

Table~\ref{tab:results_main} presents the performance of various LLMs across discrete, continuous, and free-text tasks. We report three key measures: traditional LLM-human alignment measures (accuracy, Pearson’s correlation, and similarity), the winning rate (WR, denoted as $\omega$), and the average advantage probability (AP, denoted as $\rho$). 
For each dataset, we selected $\varepsilon$ values based on the type of annotators (as indicated in Table~\ref{tab:iaa_main}): experts ($\varepsilon=0.2$), skilled annotators ($\varepsilon=0.15$), and crowd-workers ($\varepsilon=0.1$). See the discussion in \S\ref{sub:threshold} for an explanation of these choices. Below, we summarize our main findings:

\input{tables/results_advanced_boots}

\paragraph{LLMs can sometimes replace humans.} 
Table~\ref{tab:results_main} shows that many LLMs pass the alt-test across various datasets. While in two datasets (MT-Bench, and SummEval), none of the LLMs pass the test, in four (Framing, CEBAB-A, CEBaB-S and Lesion), almost all LLMs achieve $\omega \geq 0.5$. In the free-text dataset KiloGram, only Gemini-Flash passes the test. The results suggest that \textit{in many scenarios, employing LLMs can be an alternative to recruiting additional human annotators.}

However, this positive news does not imply that LLMs can always replace human annotators. The success of LLMs is nuanced and aspect-dependent. In Table~\ref{tab:results_partitions} in the Appendix, we analyze three datasets, breaking them down into sub-annotation tasks corresponding to different aspects. For instance, in the SummEval dataset (which will be discussed later), summary annotations are divided into four aspects: coherence, consistency, fluency, and relevance. Notably, each aspect may require varying levels of expertise and capabilities, and indeed, the performance of LLMs varies accordingly.

In the Lesion dataset, which involves annotating five aspects of skin lesion images, all LLMs pass our test on color-related aspects (e.g., identifying the number of colors or the presence of a bluish glow) but struggle with shape-related aspects, such as assessing asymmetry or border irregularity. In the LGBTeen dataset, all LLMs excel in the sensitivity aspect, while for five other aspects (out of ten), only one or two LLMs pass the test. In the remaining four aspects, all LLMs fail. Notably, the aspects where LLMs struggle often require higher emotional intelligence or contextual understanding (e.g., the Mental and Completeness aspects; see \citet{lissak-etal-2024-colorful}). Finally, in SummEval, most LLMs pass the test for two aspects, Coherence and Relevance, but fail on the other two. 

Our results demonstrate that test success depends on the dataset and annotation aspect, with LLMs often failing to pass it. This emphasizes the relevance of the alt-test: researchers cannot simply rely on LLM annotations without justifying this choice. 

\paragraph{Traditional measures strongly correlate with the average advantage probability.}
In addition to the statistical procedure, our method enables comparing LLM judges using the average advantage probability, $\rho$. In subsection \S\ref{sub:comparing}, we outlined the desired properties of $\rho$, such as its interpretability (as it directly represents the likelihood of the LLM being as good as or better than a random annotator) and its flexibility, allowing it to be applied to various types of annotation tasks. 

Notably, in almost all datasets, the top-ranked LLM is the same based on $\rho$ values and the traditional measures. Furthermore, in discrete tasks, the ranking of models based on Accuracy and $\rho$ shows a strong correlation, with an average Kendall $\tau$ value of 0.92. Other tasks also correlate highly, with an average Kendall $\tau$ value of 0.84, except for SummEval, which shows a negative correlation. We discuss this anomaly in Appendix \ref{sub:summeval}, which can be partially attributed to label imbalance (see Appendix~\ref{sub:imbalanced} for a solution to handling imbalance)

\input{figures/n_items_latex}

\paragraph{Few-Shot improves LLM-human alignment.}
Table~\ref{tab:results_main} indicates that the closed-source LLMs (GPTs and Geminis), outperform open-source LLMs.\footnote{Further experiments across varying model sizes are necessary to support broader claims about model openness.}
In discrete tasks, GPT-4o and Gemini-Pro consistently are the best-performing LLMs, while in continuous tasks, no single model emerges as the clear winner. However, Table~\ref{tab:results_main} reports only zero-shot experiments. Thus, we also conducted experiments using three other strategies: few-shot, CoT, and ensemble. The results are presented in Table~\ref{tab:results_advanced} and are based on 100 bootstraps of three annotators and 100 randomly sampled instances from five datasets. The reduced sample size was chosen to minimize computational costs\footnote{We annotated a maximum of 300 instances per dataset, which were then used for bootstrapping.} and primarily to reflect practical constraints better, as researchers are unlikely to annotate thousands of instances for testing whether the LLM is a good judge.

As shown in Table~\ref{tab:results_advanced}, the few-shot approach (with four demonstrations) improved the performance of nearly all LLM judges. Importantly, two few-shot LLMs achieved $\omega \ge 0.5$ on SummEval, a result not observed in the zero-shot setting. This success can be attributed to the demonstrations in the prompt, which helped align the LLMs' scoring distributions more closely with the human distributions. In contrast, the CoT methodology led to a decline in performance in many cases (45\%). Finally, the ensemble method did not improve the few-shot approach without ensembling.

\subsection{The Number Of Instances}
\label{sub:instances}

To help researchers reduce the costly need for manual annotations, we propose a statistical procedure that requires only a subset of such annotations and can verify whether an LLM can be used instead. This naturally leads to the question: how many annotated instances are needed for a reliable test? 
To answer this, we present a bootstrap analysis in Figure~\ref{fig:n_items} illustrating how the number of instances impacts our measures for the best-performing LLM (according to $\rho$) in each dataset. 

As shown, the winning rate $\omega$ strongly depends on the number of instances. This is because $\omega$ reflects the number of rejected hypotheses (i.e., the number of annotators the LLM wins), and more instances increase the power of the statistical test and the likelihood of rejecting a false null hypothesis (the human wins). In contrast, since $\rho$ does not involve hypothesis testing, it is not affected \textit{on expectation} by the number of instances. Yet, increasing the number of instances reduces the variance of $\rho$ (since it is a mean of means), making it a more robust measure for comparing LLM judges. 

Regarding the recommended number, beyond the minimum requirement of 30 instances to satisfy the normality assumption of the $t$-test, Figure~\ref{fig:n_items} shows that for $\varepsilon=0.2$, in most cases, the LLM begins to pass the test before annotating 100 instances, and in half even before 50 instances. With $\varepsilon=0.1$ the alt-test requires more instances, typically double the amount needed for $\varepsilon=0.2$, between 100 and 150. Yet, in three datasets (LGBTeen, MT-Bench, and SummEval), the LLM fails to pass the test regardless of the number of instances. While the exact number may vary depending on the task, the number of annotators, and the $\varepsilon$ value, our analysis highlights a promising finding: \textit{only a modest subset of annotations is required.} 

Finally, we refer readers to the simulation-based analysis in Appendix~\ref{app:simulations}, which provides intuition on the number of instances required under different conditions, such as the number of categories, and the reliability of the annotators or the LLM.

%% file: tables/results_advanced_boots.tex
\begin{table*}[!t]
\centering
\large
\begin{adjustbox}{width=0.98\textwidth}
\begin{tabular}{l|ccc|ccc|ccc|ccc|ccc}
\toprule
\multicolumn{16}{c}{\textbf{3 Annotators and 100 Instances Subsets} (mean values computed over 100 bootstraps)} \\
\midrule
& \multicolumn{3}{c|}{\textbf{WAX} ($\varepsilon=0.1$)} & \multicolumn{3}{c|}{\textbf{LGBTeen} ($\varepsilon=0.2$)} & \multicolumn{3}{c|}{\textbf{MT-Bench} ($\varepsilon=0.2$)} & \multicolumn{3}{c|}{\textbf{SummEval} ($\varepsilon=0.2$)} & \multicolumn{3}{c}{\textbf{10K Prompts} ($\varepsilon=0.15$)} \\
\midrule
& \underline{Acc} & \underline{WR $\omega$} & \underline{AP $\rho$} & \underline{Acc} & \underline{WR $\omega$} & \underline{AP $\rho$} & \underline{Acc} & \underline{WR $\omega$} & \underline{AP $\rho$} & \underline{Pears} & \underline{WR $\omega$} & \underline{AP $\rho$} & \underline{Pears} & \underline{WR $\omega$} & \underline{AP $\rho$} \\
Gemini-Flash & 0.37 &                    0.08 &          0.66 &                      0.55 &                     0.02 &                    0.74 &                     0.63 &                     0.0 &                   0.72 &  0.47 &                      0.0 &                    0.48 &                          0.36 &                       0.09 &                       0.66 \\
\null\quad +  4-shots & 0.41 &                    0.19 &           0.70 &                      0.66 & \cellcolor{green!30}0.61 &           \textbf{0.83} &                     0.61 &                     0.0 &                   0.73 &  0.60 &                     0.41 &                    0.76 &                          0.40 &   \cellcolor{green!30}0.58 &                       0.76 \\
\null\quad + CoT & 0.38 &                    0.09 &          0.69 &                      0.47 &                      0.0 &                     0.70 &                     0.63 &                    0.01 &                   0.76 &  0.47 &                      0.0 &                    0.46 &                          0.37 &                       0.01 &                       0.61 \\
\midrule
      Gemini-Pro & 0.40 &                    0.15 &           0.70 &                      0.50 &                      0.0 &                    0.69 &                     0.62 &                    0.01 &                   0.76 &  0.42 &                      0.0 &                    0.43 &                          0.28 &                       0.01 &                       0.61 \\
\null\quad + 4-shots & 0.39 &                    0.17 &          0.69 &                      0.55 &                     0.04 &                    0.73 &                     0.63 &                    0.03 &                   0.77 &  0.57 & \cellcolor{green!30}0.59 &                    0.77 &                          0.24 &                        0.0 &                        0.60 \\
\null\quad + CoT & 0.36 &                    0.09 &          0.68 &                      0.48 &                      0.0 &                     0.70 &                     0.58 &                     0.0 &                   0.76 &  0.49 &                      0.0 &                    0.56 &                          0.32 &                       0.01 &                       0.64 \\
\midrule
          GPT-4o & 0.37 &                    0.17 &          0.69 &                      0.65 & \cellcolor{green!30}0.55 &                    0.82 &                     0.69 &                    0.16 &                   0.78 &  0.52 &                      0.0 &                    0.49 &                          0.41 &                       0.27 &                       0.73 \\
\null\quad + 4-shots & 0.39 &                    0.15 &          0.69 &                      0.55 &                     0.03 &                    0.75 &                     0.66 &                    0.13 &                   0.78 &  0.58 &                     0.28 &                    0.74 &                          0.38 &                       0.16 &                       0.72 \\
\null\quad + CoT & 0.37 &                    0.11 &           0.70 &                      0.65 &                     0.43 &                    0.81 &                     0.65 &                     0.4 &          \textbf{0.79} &  0.58 &                     0.03 &                    0.67 &                          0.37 &                       0.43 &                       0.74 \\
\midrule
     GPT-4o-mini & 0.27 &                     0.0 &          0.59 &                      0.59 &                      0.1 &                    0.78 &                     0.60 &                     0.0 &                   0.73 &  0.49 &                      0.0 &                    0.53 &                          0.36 &                       0.48 &                       0.76 \\
\null\quad + 4-shots & 0.30 &                    0.01 &          0.62 &                      0.60 &                     0.12 &                    0.77 &                     0.61 &                     0.0 &                   0.74 &  0.60 & \cellcolor{green!30}0.77 &           \textbf{0.79} &                          0.42 &   \cellcolor{green!30}0.74 &              \textbf{0.78} \\
\null\quad + CoT & 0.33 &                     0.0 &          0.66 &                      0.57 &                     0.06 &                    0.75 &                     0.59 &                     0.0 &                   0.72 &  0.56 &                      0.0 &                     0.60 &                          0.32 &                       0.44 &                       0.74 \\
\midrule
Ens. Geminis & 0.42 &                    0.21 &          0.71 &                      0.56 &                     0.11 &                    0.77 &                     0.66 &                    0.03 &                   0.76 &  0.48 &                      0.0 &                    0.55 &                          0.33 &                       0.06 &                       0.67 \\
   Ens. GPTs & 0.38 &                    0.05 &          0.67 &                      0.61 &                     0.19 &                    0.79 &                     0.60 &                     0.0 &                   0.73 &  0.58 &                     0.04 &                    0.66 &                          0.39 &   \cellcolor{green!30}0.64 &                       0.77 \\
Ens. All & 0.44 &                    0.24 & \textbf{0.73} &                      0.63 &                     0.37 &                     0.80 &                     0.61 &                    0.01 &                   0.74 &  0.58 &                     0.02 &                    0.66 &                          0.39 &                       0.41 &                       0.74 \\
\bottomrule
\end{tabular}
\end{adjustbox}
\caption{\textbf{Results -- Advanced LLM Judges:} Each data point is calculated using a bootstrap of 100 combinations of three annotators and one hundred instances.
\textit{Ens.} stands for ``Ensemble''. Please see the caption of Table~\ref{tab:results_main}.}
\label{tab:results_advanced}
\vspace{-0.8em}
\end{table*}

%% file: figures/n_items_latex.tex
\begin{figure*}[t]
    \centering
    \includegraphics[width=0.975\textwidth]{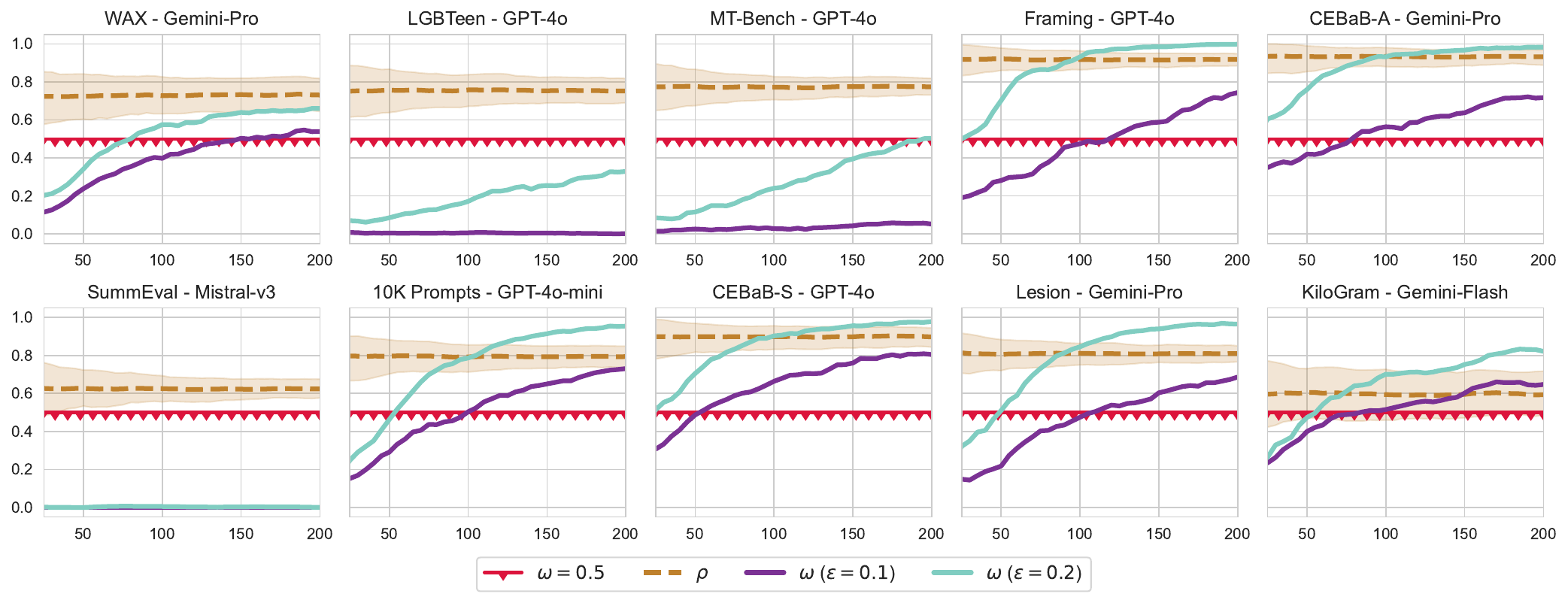}
    \caption{\textbf{Analysis of the Impact of the Number of Items:} Each data point is calculated using a bootstrap of 100 combinations of three annotators and $n$ items (x-axis). The y-axis shows the winning rates ($\omega$, solid lines) for $\varepsilon=0.1$ (purple) and $\varepsilon=0.2$ (turquoise). In addition, it presents the average advantage probability ($\rho$, dashed brown line) with its empirical 0.9 confidence intervals. The subplot title indicates the examined LLM.}
    \label{fig:n_items}
\vspace{-0.8em}
\end{figure*}

%% file: sections/conclusion.tex
\section{Conclusion}
\label{sec:conclusion}

Science advances through systematic observation, precise measurement, and the rigorous validation of hypotheses. It is no coincidence that Pearson famously claimed statistics to be \textit{``the grammar of science''}. As results and findings of studies increasingly rely on LLMs instead of human annotators, extra care is needed to uphold scientific rigor. 

In this paper, we proposed a statistical procedure to justify using LLM annotations in research studies, the alt-test, which is simple and requires minimal effort. As demonstrated in our analysis, researchers can recruit a small group of annotators (at least three) to annotate a subset of 50 to 100 examples, depending on the complexity of the task. 

Appendix~\ref{faq} provides a list of frequently asked questions about our procedure, along with answers and best practices. Then, in Appendix~\ref{sec:discussion}, we further discuss and analyze additional aspects of our procedure, like the impact of $\varepsilon$ and the choice of human annotators. Finally, in Appendix~\ref{sec:advanced}, we propose modifications to our procedure to address advanced scenarios: handling imbalanced labels (\S\ref{sub:imbalanced}), benchmarking against a single expert annotator (\S\ref{sub:single_expert}), incorporating annotator quality scores (\S\ref{sub:quality}), respecting minotiy opinions in subjective annotation tasks (\S\ref{sub:subjective}), and testing whether LLMs outperform humans (\S\ref{sub:gold_label}). 

We encourage researchers to adopt our procedure to ensure more reliable and transparent evaluations of LLMs, and careful practices to leverage their annotations in NLP research and other fields.

%% file: sections/limitations.tex
\section{Limitations}

\paragraph{Data contamination}
One limitation of our experiments is the potential for data contamination, where datasets used in our experiments may overlap with the training data of the evaluated LLMs. Popular datasets such as SummEval and MT-Bench, commonly used for benchmarking LLM-as-judges, are publicly available and might have been included in the training data of some LLMs. Notice that most of the datasets we used are recent (published after 2022) and not widely known, with fewer than 50 citations each. Additionally, one of our datasets, LGBTeen, is available only upon request. Hopefully, this lowers the risk of data contamination.

\paragraph{High disagreement among human annotators}
High disagreement among human annotators can arise from various factors, such as untrained crowd workers, annotators who are not suited for the task, unclear or poorly designed annotation guidelines, or the inherently subjective nature of the task itself. In such cases, and as demonstrated in our simulation-based analysis in Appendix~\ref{app:simulations}, it is less likely that the LLM-as-a-judge will succeed in passing our test. The procedure compares the LLM with each annotator to test alignment with the remaining annotators. When the remaining annotators are inconsistent, this introduces high variance in determining who aligns better (the LLM or the excluded annotator). Under these conditions, the hypothesis test is unlikely to reject the null hypothesis, and the LLM's winning rate remains low.

This property of our procedure can be desirable, as it may help researchers identify potential issues with the annotation process, such as unclear guidelines, unqualified annotators, or the inherent subjectivity of the task. Traditional measures would similarly yield low scores in such cases.

For inherently subjective tasks, we advocate for developing alternative methods to assess the quality of human annotations, where disagreements are a feature rather than a flaw \citep{basile2021we, uma2021learning} and methods to evaluate the LLM-as-a-judge's ability to represent a spectrum of opinions. Finally, we refer readers to \S\ref{sub:subjective} in the Appendix, where we discuss modifications of our procedure to better account for subjectivity and emphasize minority opinions.

\paragraph{Comparing against weak human annotators} 
A potential misuse of our procedure is intentionally comparing the LLM against weak human annotators to demonstrate that the LLM outperforms them and justify its use. In cases where human annotators are intentionally weak, with low inter-annotator agreement, the LLM might pass the test, as shown in our simulation-based analysis in Appendix~\ref{app:simulations}. To ensure sound and transparent testing, researchers should always report the IAA of the human annotators. If the IAA is low, the conclusions drawn from the alt-test are less reliable, and to compensate for this, researchers must use small values of $\varepsilon \le 0.1$ and annotate more instances.

In the single expert scenario (see Appendix~\ref{sub:single_expert}), the LLM is compared against non-experts, and both are tested for alignment with a single expert. If the non-experts are particularly weak (e.g., inconsistent or unqualified), the LLM may appear to outperform them, and our procedure cannot fully prevent such misuse. Science, however, is built on transparency and trust. We strongly encourage researchers to disclose detailed information about the annotators and to publish the human annotations, allowing others to reproduce and validate the results. As discussed in \S\ref{sec:discussion}, the expertise of the human annotators directly impacts the reliability and authority of the procedure. Readers and reviewers should critically assess the choice of annotators, and if the annotators are deemed unsuitable, the study’s results should be taken with a grain of salt.

%% file: sections/arr/appendix_arr.tex
\input{sections/faq}
\input{sections/arr/discussion_arr}

\input{sections/simulations}
\input{sections/advanced_appendix}
\input{sections/theorem}




\section{Datasets}
\label{app:dataset_overview}

\begin{itemize}
    \item \textbf{WAX} \cite{liu2022wax} -- Prompt provided in Box~\ref{box:wax}. 
    We use the Relation Labeling task from the Word Association eXplanations (WAX) dataset. In this task, MTurk annotators were presented with two words—a cue word and an associated word (e.g., \textit{shark} and \textit{sharp}), along with an explanation (e.g., ``shark teeth are sharp''). The annotators labeled the relation between the two associated words based on the given explanation, selecting from 16 predefined relation types. We included only items that were annotated by at least five crowd workers.
    \item \textbf{SummEval} \cite{fabbri2021summeval} -- Prompt provided in Box~\ref{box:summeval}.
    This dataset includes human evaluations of summaries generated by 16 neural summarization models applied to 100 documents from the CNN/DailyMail test set. We focused on expert annotations (authors of summarization papers) collected for four dimensions: coherence, consistency, fluency, and relevance. The annotators rated summaries on a Likert scale from 1 to 5, with higher scores indicating better quality. 
    \item \textbf{LGBTeen} \cite{lissak-etal-2024-colorful} -- Prompt provided in Box~\ref{box:lgbteen}. 
    Three expert annotators evaluated responses from humans and various LLMs to queries from queer youth, extracted from the r/LGBTeen subreddit. Each response was assessed using a ten-question questionnaire designed to evaluate desirable traits, such as inclusiveness, sensitivity, and openness (see Box~\ref{box:lgbteen_q}). Responses were categorized as `Yes,’ `Partially,’ `No,’ or `Irrelevant’. We kept only responses that were annotated by at least two annotators. 
    \item \textbf{MT-Bench} \cite{mtbench} -- Prompt provided in Box~\ref{box:mtbench}.
    MT-Bench is a dataset consisting of 80 manually crafted multi-turn questions designed to evaluate the conversational and instruction-following abilities of LLMs. The dataset covers eight categories of prompts, such as writing, reasoning, math, and coding. Expert annotators, including the paper's authors and graduate students with expertise in the relevant categories, evaluated responses from LLMs by assessing 20 multi-turn questions conversation. For each question, annotators selected the better response between two competing LLM responses or marked it as a tie. We included only items annotated by at least two annotators and annotators who evaluated more than 30 items.
    \item \textbf{Lesion} \cite{Cheplygina2018} -- Prompt provided in Box~\ref{box:lesion}.
    This dataset includes images of skin lesions from the ISIC 2017 challenge \citep{CodellaGCHMDKLM18} that undergraduate students annotated during a project on medical image analysis. Each image was annotated with five features: asymmetry (scale 0-2), irregularity of the border (0-2), number of colors present (1-6), presence of structures such as dots (0-2) and presence of a blueish glow (0-2).
    \item \textbf{Framing} \cite{frermann-etal-2023-conflicts} -- Prompt provided in Box~\ref{box:framing}.
    This dataset consists of articles on climate change annotated with 22 yes/no questions about narrative framing. The questions are grouped into five framing categories: resolution, conflict, human interest, moral, and economic. The 22 questions and annotation guidelines are presented in Boxes~\ref{box:framing_q} and \ref{box:framing_g}. The annotations were performed by four on-site annotators with backgrounds in social and political sciences, who underwent an extensive training phase. We included only article-question pairs that were annotated by at least three annotators.
    \item \textbf{CEBaB} \cite{abraham2022cebab} -- Prompt provided in Box~\ref{box:cebab}.
    This large-scale dataset comprises restaurant reviews annotated by crowd workers. The workers labeled the sentiment of four aspects: Food, Service, Noise, and Ambiance. Each aspect was categorized as `Positive', `Negative' or `Unknown'. Additionally, star ratings were provided on a five-point scale. We use two variants of this dataset: \textit{CEBaB-A}, which includes annotations for the four aspects, and \textit{CEBaB-S}, which includes the star ratings. For each variant, we retained only items annotated by at least three annotators. We identified a subset of ten annotators with the highest overlap of annotated items (i.e., items annotated by the largest number of these ten annotators).
    \item \textbf{10K Prompts}\footnote{\url{https://huggingface.co/datasets/data-is-better-together/10k_prompts_ranked}} -- Prompt provided in Box~\ref{box:10k_prompts}.
    This dataset is part of a project by Argilla and HuggingFace and was created by collecting prompts from various sources. The annotators are members of the HuggingFace community tasked with ranking the quality of synthetic and human-generated prompts on a Likert scale from 1 to 5. We identified a set of 13 annotators, each with at least 30 items also annotated by another annotator.
    \item \textbf{KiloGram} \cite{ji-etal-2022-abstract} -- Prompt provided in Box~\ref{box:kilogram}. 
    This dataset includes thousands of tangram images (see an example in Figure~\ref{fig:tangram}), annotated by MTurk workers. Each annotator provided a short free-text description of what the tangram shape looks like. For computing similarity between annotations, we use cosine similarity applied to representations extracted by a SentenceTransformer model. Note that we tested various SentenceTransformer models based on the HuggingFace STS English leaderboard\footnote{\url{https://huggingface.co/spaces/mteb/leaderboard}}, and the results presented in Table~\ref{tab:results_kilogram}. We decided to report the results using `e5-large-v2'.\footnote{\url{https://huggingface.co/intfloat/e5-large-v2}}
\input{figures/tangram_latex}
\input{tables/kilogram_models}
\end{itemize}

\onecolumn

\section{Additional Results}
\label{app:additional_results}

\input{tables/results_partitions}
\input{tables/summeval_dist}
\input{sections/prompts}

%% file: sections/faq.tex
\section{Frequently Asked Questions}
\label{faq}

\medskip\noindent\textbf{Q: How should I report the alt-test results?}
\\ \textcolor{Green}{\textbf{A:}} We recommend the following best practices for applying and reporting the alt-test results: 
\begin{enumerate}[leftmargin=*]
    \item Provide details about the human annotators, including their profile, level of expertise, annotation guidelines, training, and the overall process.
    \item Explain the rationale behind the choice of $\varepsilon$ (see the relevant question below for guidance).
    \item For selecting the number of instances, see the relevant question below.
    \item Report a measure of reliability for the human annotators, such as inter-annotator agreement (e.g., Cohen’s $\kappa$) or correlation measures. This is essential to ensure that the annotators are sufficiently reliable and the $\varepsilon$ value is appropriate.
    \item For selecting the LLM-as-a-judge, report the average advantage probability ($\rho$), clearly state which LLMs are compared, and provide their corresponding $\rho$ values.
    \item Report the winning rate of the selected LLM.
\end{enumerate}

\medskip\noindent\textbf{Q: Why not use an Inter-Annotator Agreement (IAA) measure?}
\\ \textcolor{Green}{\textbf{A:}} Our procedure is a type of IAA, but unlike traditional IAA measures (such as Cohen’s kappa), which assess agreement among a group of annotators, our goal is to \textit{compare} the LLM to the group to determine whether it can replace them.

\medskip\noindent\textbf{Q: Why not use a traditional measure such as F1 score or accuracy?}
\\ \textcolor{Green}{\textbf{A:}} To compare the LLM to human annotators and to address the `replacement question' (i.e., whether the LLM can be used instead of the annotators), one might consider traditional LLM-human alignment measures (e.g., the F1 score or a correlation between the LLM and the majority vote label). However, answering the replacement question requires statistical rigor. Even though a statistical test can check if the traditional measure exceeds a predefined threshold, there is no universal standard for setting it, which may vary across datasets and setups. Additionally, traditional measures only evaluate whether the LLM matches human performance, not whether it provides a better alternative.

In contrast, our procedure involves statistical practices and provides clear passing criteria. Most importantly, it directly answers the replacement question by using a leave-one-out approach -- excluding one annotator at a time and assessing whether the LLM better represents the remaining annotators than the excluded one.

\medskip\noindent\textbf{Q: Why do you recommend at least three human annotators and not two?}
\\ \textcolor{Green}{\textbf{A:}} While our procedure can be used with two annotators, we believe it is less reliable. With only two, the procedure simply checks whether the LLM aligns more with one annotator than the other, lacking a consensus signal. This makes results more sensitive to individual biases. With at least three annotators, the procedure better evaluates whether the LLM represents the broader group. Obviously, the more annotators, the better, as this increases the reliability, reduces the influence of individual biases, and provides a more robust consensus signal.

\medskip\noindent\textbf{Q: What if I have annotations from a single human annotator?}
\\ \textcolor{Green}{\textbf{A:}} Since our procedure requires at least two annotators, we recommend recruiting additional annotators for the alt-test. However, if the single annotator is an expensive expert (or you trust their annotations) and cannot recruit others at the same expertise level, you can instead recruit lower-quality annotators and test who better represents the expert: the LLM or the newly recruited annotators. We refer to this as the single-expert scenario and provide a detailed discussion in Appendix~\ref{sub:single_expert}.

\medskip\noindent\textbf{Q: How do I select the $\varepsilon$ value?}
\\ \textcolor{Green}{\textbf{A:}} We discuss this topic in detail in \S\ref{sub:threshold}. Note that $\varepsilon$ is the cost-benefit hyperparameter, where higher values indicate greater efficiency advantages of the LLM. As a rule of thumb, for expert annotators (reliable but expensive, sometimes inaccessible), set $\varepsilon=0.2$. For skilled annotators (e.g., undergraduate students, trained workers, etc., who are less reliable than experts), set $\varepsilon=0.15$. For crowd-workers, set $\varepsilon=0.1$. Moreover, the choice of $\varepsilon$ should depend on the reliability of the human annotators. When IAA is low, a smaller $\varepsilon$ should be used. The simulation-based analysis in Appendix~\ref{app:simulations} can help understand the effect of IAA on the alt-test, and guide the selection of an appropriate $\varepsilon$.

\medskip\noindent\textbf{Q: How many instances should I annotate?}
\\ \textcolor{Green}{\textbf{A:}} We discuss this topic in detail in \S\ref{sub:instances}. To ensure the normality assumption of the t-test holds, you should have at least 30 instances. Our analysis shows that annotating between 50 and 100 instances is sufficient in most cases. Obviously, the more annotated instances, the better, as this increases the statistical power of the t-test and the likelihood of the LLM passing the alt-test. We encourage researchers to conduct simulation analyses similar to the one presented in Appendix~\ref{app:simulations} to help determine the required number of instances. The simulation code is available in our GitHub repository. It can be customized by adjusting parameters such as the number of categories or the expected IAA to reflect the characteristics of their data.

\medskip\noindent\textbf{Q: What if I have fewer than 30 annotated instances per annotator?}
\\ \textcolor{Green}{\textbf{A:}} In this case, the normality assumption of the t-test does not hold, so a non-parametric test, such as the Wilcoxon signed-rank test, should be used instead. Still, we strongly recommend having annotators label additional instances. See the next question for an alternative approach.

\medskip\noindent\textbf{Q: I have two sets of human annotators. Can I combine annotators from the first set with the second set to increase the number of instances per annotator?}
\\ \textcolor{Green}{\textbf{A:}} 
If you have two separate sets of annotators who annotated different, non-overlapping instances, you can artificially increase the number of instances per annotator by pairing them across sets. 
For example, suppose Set 1 consists of three annotators who annotated 20 instances, and Set 2 consists of another three annotators who annotated a different set of 20 instances. You can combine an annotator from Set 1 with an annotator from Set 2, treating them as a single ``combined annotator'' with 40 instances. To improve robustness, you can form multiple such pairs and report the average winning rate across different pairing combinations.

While this approach can increase the number of annotated instances per annotator, it is not ideal. The best practice is still to annotate more instances. Combining annotators like this may also increase the variance of the statistics (since we combine instances annotated by different distributions). This could lead to higher p-values, making the LLM fail.


\medskip\noindent\textbf{Q: What if I care about ranking rather than exact scores?}
\\ \textcolor{Green}{\textbf{A:}} In some cases, the exact match between LLM predictions and human annotations may not be as important as the relative ordering of instances. For example, if the goal is to ensure that higher-scored instances by humans are also ranked higher by the LLM. To evaluate this, we can adapt our procedure to operate on ranks instead of raw scores. Specifically, we create a separate ranked list for each human annotator and the LLM by assigning ranks to instances based on their annotated scores (e.g., the lowest score gets rank 1). We then apply our procedure to these ranks, replacing the original annotations. The alignment scoring function can be negative RMSE, computed for each instance based on the difference between its rank assigned by the LLM and its rank assigned by the human annotator.

\medskip\noindent\textbf{Q: What if I have a skewed label distribution?}
\\ \textcolor{Green}{\textbf{A:}} In Appendix~\ref{sub:imbalanced}, we discuss modifications to our procedure to account for label imbalance.

\medskip\noindent\textbf{Q: How to test if the LLM can be used in several environments or domains?}
\\ \textcolor{Green}{\textbf{A:}} When evaluating whether an LLM-as-a-judge can be used across multiple environments or domains, it is important to evaluate it in each setting independently while also controlling for the overall False Discovery Rate (FDR). For example, suppose we have five domains, each with three human annotators, resulting in 15 comparisons between the LLM and humans. The FDR-controlling procedure should be applied to the 15 p-values to ensure statistical rigor. Additionally, the winning rate should be computed separately for each environment, and the results should be summarized as: \\
\textit{``The LLM passes the alt-test in X out of 5 domains.''}

In cases of hundreds of environments, collecting labeled data from at least three annotators per environment may be impractical. This remains an open challenge, but it offers promising directions for future work, such as sampling representative environments rather than testing all of them.

\medskip\noindent\textbf{Q: How to test who better represents human experts? LLMs or crowd-workers?}
\\ \textcolor{Green}{\textbf{A:}} We discuss this scenario in Appendix~\ref{sub:single_expert}.

\medskip\noindent\textbf{Q: How to test whether LLMs outperform humans?} (and not whether they can replace them)?
\\ \textcolor{Green}{\textbf{A:}} We discuss this scenario in Appendix~\ref{sub:gold_label}.

\medskip\noindent\textbf{Q: What if I trust one annotator more than the others?}
\\ \textcolor{Green}{\textbf{A:}} In Appendix~\ref{sub:quality}, we discuss simple modifications to our procedure to account for variations in annotator quality.

%% file: sections/arr/discussion_arr.tex
\section{Discussion}
\label{sec:discussion}

The goal of this section is to discuss factors that influence the outcomes of the alt-test: the number of annotated instances (which was already discussed in \S\ref{sub:instances}), the value of the cost-benefit trade-off hyperparameter $\varepsilon$ (\S\ref{sub:threshold}), and the profile of the human annotators against whom we compare the LLM (\S\ref{sub:humans}). In addition, we also present a case study analysis of the SummEval dataset (\S\ref{sub:summeval}).

\input{figures/epsilon_analysis_latex}

\subsection{The Cost-benefit Hyperparameter}
\label{sub:threshold}

We wish to use LLMs instead of human annotators since they offer a much cheaper, faster, and less labor-intensive alternative. Therefore, we incorporated a cost-benefit hyperparameter into our procedure, $\varepsilon$, which lowers the necessary threshold the LLM must exceed (i.e., $\rho_j^h - \varepsilon$) to pass the alt-test. Generally, higher values of $\varepsilon$ are recommended when the cost and labor savings provided by the LLM are substantial. For instance, this applies when human annotators are highly expensive, require extensive and prolonged training, or when the task is time-consuming or particularly challenging (e.g., annotating complex relationships within lengthy documents). Conversely, smaller values of $\varepsilon$ are more appropriate for simple annotation tasks that untrained crowd-workers can complete.

To explore the relationship between different $\varepsilon$ values and the outcomes of the alt-test, as well as to provide guidelines for setting these values, we analyze the effect of $\varepsilon$ on the winning rate $\omega$ of four LLMs, as shown in Figure~\ref{fig:thresholds}. The strong monotonic increasing relationship between $\varepsilon$ and $\omega$, as presented by our analysis, enables us to identify the effective range of $\varepsilon$, which lies between 0.05 and 0.3. For $\varepsilon > 0.3$, all LLMs achieve $\omega \ge 0.5$ on every dataset (except SummEval, and Gemini-Pro in KiloGram) and pass the test. In contrast, for $\varepsilon < 0.05$, all LLMs achieve $\omega < 0.5$ on all datasets (except CEBaB-S) and fail the test.

From this analysis, we derive practical guidelines for selecting appropriate $\varepsilon$ values. First and foremost, any value can be valid if the researcher reasonably justifies their choice. This justification may involve several aspects, including the cost and effort of the annotation, the expertise of the annotators, the cost of annotation mistakes (which varies based on the application and domain), and the centrality of LLM annotations to the study.Moreover, based on the simulation-based analysis in Appendix~\ref{app:simulations}, we recommend selecting $\varepsilon$ values according to the quality of the human annotators. When annotator reliability is low (e.g., low IAA), a smaller $\varepsilon$ should be used. This aligns with the expectation that expert annotators tend to be more reliable than skilled annotators, who in turn are generally more reliable than crowd-workers.

As a rule of thumb, we recommend setting $\varepsilon$ to 0.2 when the annotators are trusted experts or highly reliable, and 0.15 when they are skilled annotators (e.g., undergraduate students or trained workers). If the annotators are crowd workers or have low reliability, $\varepsilon$ should be set to 0.1. In either case, the quality of the annotators must be high enough to ensure reliable annotations, as discussed in the following subsection. In our experiments, we selected $\varepsilon$ values based on the type of annotators (as indicated in Table~\ref{tab:iaa_main} and Figure~\ref{fig:thresholds}) and the recommendations above. 

\subsection{The Human Annotators Profile}
\label{sub:humans}

Recall that our procedure aims to justify replacement if \textit{the LLM aligns more closely with the collective distribution than an individual does}, where the collective distribution approximates the gold label distribution. This collective distribution is the most reliable and authoritative benchmark when the annotators are experts. Accordingly, we recommend using expert annotators whenever possible and, at the very least, highly trained crowd-workers. If researchers themselves are experienced with the task, they can serve as annotators. 

In \S\ref{sec:advanced}, we examine advanced topics related to human annotators. In \S\ref{sub:single_expert}, we address the scenario of a single expert annotator and propose a simple modification to our procedure. This scenario is particularly relevant when only one expert is available due to limited accessibility or the high cost of their annotations. This single expert annotates a small subset of instances, and their annotations are considered the gold labels (i.e., there is no collective distribution in this scenario). Our modification compares the LLM against non-experts to determine whether the LLM aligns more closely with the single expert than a non-expert does.

Additionally, in \S\ref{sub:quality}, we propose a modification to our procedure that incorporates a quality score for each human annotator. This score can be derived from various sources, such as qualification tests, and allows researchers to account for annotator expertise and reliability differences.

In \S\ref{sub:subjective}, we address the unique challenges of subjective annotation tasks, where minority opinions may carry importance. For example, in hate speech and offensive language detection, it is often a single sensitive annotator, frequently from an underrepresented group, who identifies the offensive content and deviates from the majority label. In such cases, we aim to adapt our method to account for and emphasize minority votes.

Finally, many studies aim not to use LLMs for annotations or judgments but to evaluate whether LLMs outperform humans. For example: \textit{``ChatGPT Out-scores Medical Students on Clinical Care Exam Questions''} \citep{stanford2023chatgpt}. In these cases, gold labels (e.g., exam answers) are available and are used for benchmarking. Moreover, we set $\varepsilon=0$ because there is no need to penalize humans. In \S\ref{sub:gold_label}, we discuss adapting the alt-test to rigorously answer if LLMs outperform humans.

\subsection{Case study: SummEval}
\label{sub:summeval}

Table~\ref{tab:results_main} reveals an anomaly in the SummEval dataset: Mistral-v3 achieves the highest $\rho$. Interestingly, Mistral’s traditional measure score (Pearson’s correlation) is low (0.12). This discrepancy warrants further investigation. As shown in Table~\ref{tab:results_partitions} in the Appendix, Mistral passes the test only for the Consistency aspect, with $\rho=0.87$, much higher than other LLMs (around 0.45). 

First, this demonstrates why each aspect should be tested separately.
Second, Table~\ref{tab:summeval_dist} in the Appendix, which reports the annotation distributions for SummEval, explains why Mistral's $\rho$ is so high: human annotations for Consistency are highly skewed, with the score ‘5’ assigned 89\% of the time. The only LLM with a similarly skewed prediction distribution is Mistral. Other LLMs predict ‘5’ only about 30\% of the time. However, as shown by Table~\ref{tab:summeval_dist}, few-shot helps LLMs adjust and skew their distributions, improving their alignment.

Noteworthy, unlike traditional measures (Pearson's and Spearman's correlations), our method captures this nuance in alignment. In \S\ref{sub:imbalanced} of the Appendix, we discuss label imbalance (like this case) and propose an adjustment to our method using Inverse Probability Weighting (IPW).

%% file: figures/epsilon_analysis_latex.tex
\begin{figure*}[!t]
    \centering
    \includegraphics[width=0.975\textwidth]{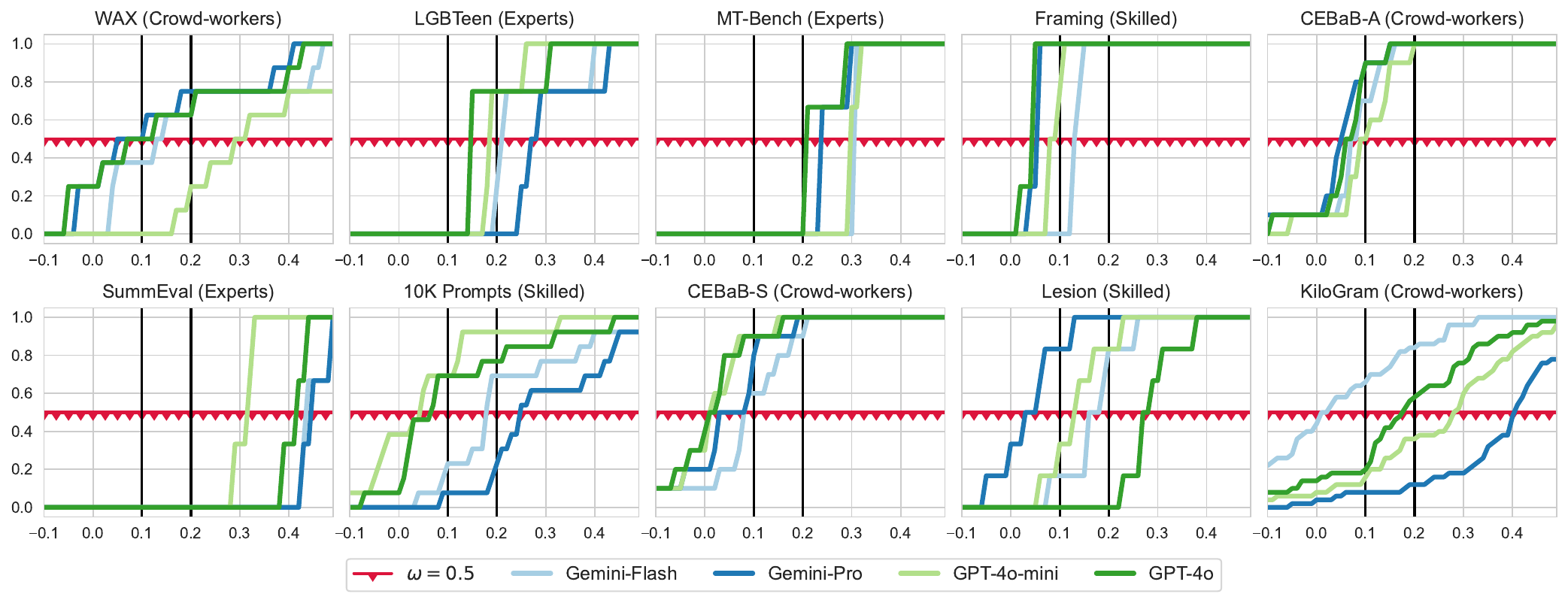}
    \caption{\textbf{Analysis of the Impact of Different $\varepsilon$ Values:} The x-axis represents different $\varepsilon$ values, while the y-axis shows the winning rate $\omega$ for four LLMs. If $\omega \ge 0.5$ (red line with triangles), the LLM passes the test, indicating it is a comparable alternative to human annotators when considering the cost-benefit tradeoff represented by $\varepsilon$. The annotator types are stated next to the dataset names.}
    \label{fig:thresholds}
\vspace{-0.8em}
\end{figure*}

%% file: sections/simulations.tex
\section{Simulations}
\label{app:simulations}

The goal of this section is to explore the behavior of the alt-test further and demonstrate that it behaves as expected. Since the available datasets do not support the fine-grained analysis we seek, we turn to simulated data. Specifically, we focus on discrete annotation tasks and simulate both LLM and human annotations by controlling the level of noise in their annotations and the number of categories (classes). By varying the noise level, we can simulate LLM or human annotators ranging from poor to accurate, allowing us to analyze how many instances are required to test the LLM.

\input{figures/simulation_noises_tex}

This simulation-based analysis can also help researchers determine how many human annotations they should collect for the alt-test, depending on their expectations about the quality of the LLM and the reliability of human annotators. The simulation code is available in our GitHub repository for further use by the research community. We begin by describing the simulation procedure and then proceed to analyze the results.

We simulate annotation data for a discrete labeling task with $n$ instances $\{x_1,\ldots,x_n\}$, $m$ human annotators $\{h_1,\ldots,h_m\}$, and an LLM $f$.
First, we draw the class prior vector over $K$ categories by sampling from a Dirichlet distribution with a symmetric parameter vector of ones, $\mathbf{1}_K=(1,...,1)$:
\begin{align*}
    \boldsymbol{\pi} &\sim \mathrm{Dirichlet}(\mathbf{1}_K)
\end{align*}
\noindent next, for each instance $x_i$, we sample its gold label:
\begin{align*}
    y(x_i) &\sim \mathrm{Categorical}\bigl(\boldsymbol{\pi}\bigr)
\end{align*}

For each human annotator $h_j$ and $i\in\{1,...,n\}$, we define a noisy annotation distribution in which with probability $1-\eta_h$ the true label $y(x_i)$ is chosen and with probability $\eta_h$ a label is drawn from  $\boldsymbol{\pi}$:
\begin{align*}
    \mathbf{p}^{h_j}_i 
      &= (1 - \eta_h)\,\mathbf{e}_{y(x_i)} \;+\;\eta_h\,\boldsymbol{\pi},\\
h_j(x_i)\;\;&\sim\;\mathrm{Categorical}\bigl(\mathbf{p}^{h_j}_i\bigr)
\end{align*}
\noindent where $\mathbf{e}_{y(x_i)}$ is the one-hot vector corresponding to the gold label $y(x_i)$. The LLM annotates every instance analogously, but with noise level $\eta_f$:
\begin{align*}
f(x_i)\;\;&\sim\;\mathrm{Categorical}\bigl(\mathbf{p}^{f}_i\bigr)
\end{align*}
The noise parameter $\eta$ controls reliability. In a task with $K=4$ categories, setting $\eta_h=0.1$ yields an IAA Cohen’s $\kappa\approx 0.8$ among human annotators, indicating high agreement.  In contrast, $\eta_h=0.5$ produces $\kappa\approx 0.2$, reflecting weak agreement.  By varying $\eta_h$ and $\eta_f$, we simulate annotators or LLMs with poor to perfect performance.

For each triplet of noise levels and number of categories $(\eta_h, \eta_f, K)$, the simulation is based on 50 independently generated datasets, each constructed according to the distributions defined above, with six human annotators and 500 instances per dataset. For each sample size considered (ranging from 30 to 200 instances), we perform 50 bootstrap samples within each dataset, by randomly selecting the specified size and three human annotators. Thus, each data point is aggregated over a total of 2,500 bootstraps (50 datasets $\times$ 50 bootstraps), providing stable and reliable estimates.

\input{figures/simulation_categories_tex}

The simulation results are presented in Figures~\ref {fig:noise_simulation} and~\ref{fig:categories_simulation}. The first figure shows how varying both $\eta_h$ and $\eta_f$ (i.e., the quality of the human annotators and the LLM) affects the behavior of the alt-test when $K=4$. The second figure shows how varying $K$ (the number of categories) and $\eta_h$ affects the alt-test when $\eta_f=0.2$. For both figures, we report the winning rate for four values of $\varepsilon$ and the average advantage probability, along with 0.9 empirical CIs. We also report the IAA Cohen’s $\kappa$, representing the quality of the human annotators and the accuracy of the LLM with the human majority vote, representing the quality of the LLM. 

\paragraph{Stronger LLM requires fewer instances.} As shown in Figure~\ref{fig:noise_simulation}, the larger the gap in noise levels in favor of the LLM (i.e., moving down within the same column of subfigures, with $\eta_f > \eta_h$), the fewer instances are needed, as desired. The LLM can pass the test even for smaller $\varepsilon$ values, including $\varepsilon = 0$, when the gap is large enough (greater than 0.3). This desirable behavior demonstrates that the alt-test reliably detects when the LLM is genuinely better than human annotators.
When they have the same noise level ($\eta_h = \eta_f$, diagonal subfigures), the LLM passes the test for $\varepsilon = 0.2$ (a relatively large value) with fewer than 50 instances, but not for stricter thresholds ($\varepsilon < 0.1$).

\paragraph{Larger noise requires more instances.}
When fixing a noise gap ($\eta_f-\eta_h$), the larger both noises are, the more instances are required for passing the alt-test. This is desirable, as higher noise requires more statistical power. As a result, the alt-test discourages comparisons between low-quality LLMs and low-quality annotators, and instead favors comparisons to high-quality human annotators.

\paragraph{Stronger human annotators require more instances.}
When the human annotations are less noisy (i.e., moving left within the same row of subfigures in Figure~\ref{fig:noise_simulation}), leading to higher reliability and greater IAA values, it becomes harder for the LLM to pass the alt-test, and more instances are required. This is expected, as high-quality annotators provide a stronger baseline. However, this should not incentivize researchers to use weak annotators intentionally. To ensure sound and transparent testing, researchers should always report the IAA of the human annotators. If the IAA is low, the conclusions drawn from the alt-test are less reliable, and to compensate for this, researchers must use small values of $\varepsilon \le 0.1$ and annotate more instances.

\paragraph{The impact of the number of categories.}
Figure~\ref{fig:categories_simulation} illustrates how the behavior of the alt-test varies with the number of categories, under a fixed LLM noise level ($\eta_f = 0.2$) and varying human annotator noise levels. The analysis shows that when human annotators are reliable ($\eta_h \le 0.2$, IAA $\ge 0.6$), increasing the number of categories requires more instances for the LLM to pass the test. In contrast, when annotators are less reliable ($\eta_h \ge 0.4$, IAA $\le 0.4$), increasing the number of categories makes it easier for the LLM to pass. This occurs because both the LLM and the excluded human annotator are more likely to predict a label that differs from the other two annotators, leading to ties in our procedure. This phenomenon is also reflected in the decreasing accuracy of the LLM with the majority vote as the number of categories increases. To compensate for this effect, we recommend using smaller values of $\varepsilon$, which is always advisable whenever human annotators are noisy, and reporting an additional traditional metric, such as accuracy with the majority vote, to complement the alt-test results.
Nevertheless, the behavior observed in the scenario of noisy human annotators combined with many categories is expected. In such a scenario, the resulting human annotations are of low quality. Since the goal of the alt-test is to assess whether the LLM is a comparable alternative to recruiting human annotators, it is appropriate that the LLM passes the test when it provides a more reliable option. 

%% file: figures/simulation_noises_tex.tex
\begin{figure*}[!t]
    \centering
    \includegraphics[width=0.975\textwidth]{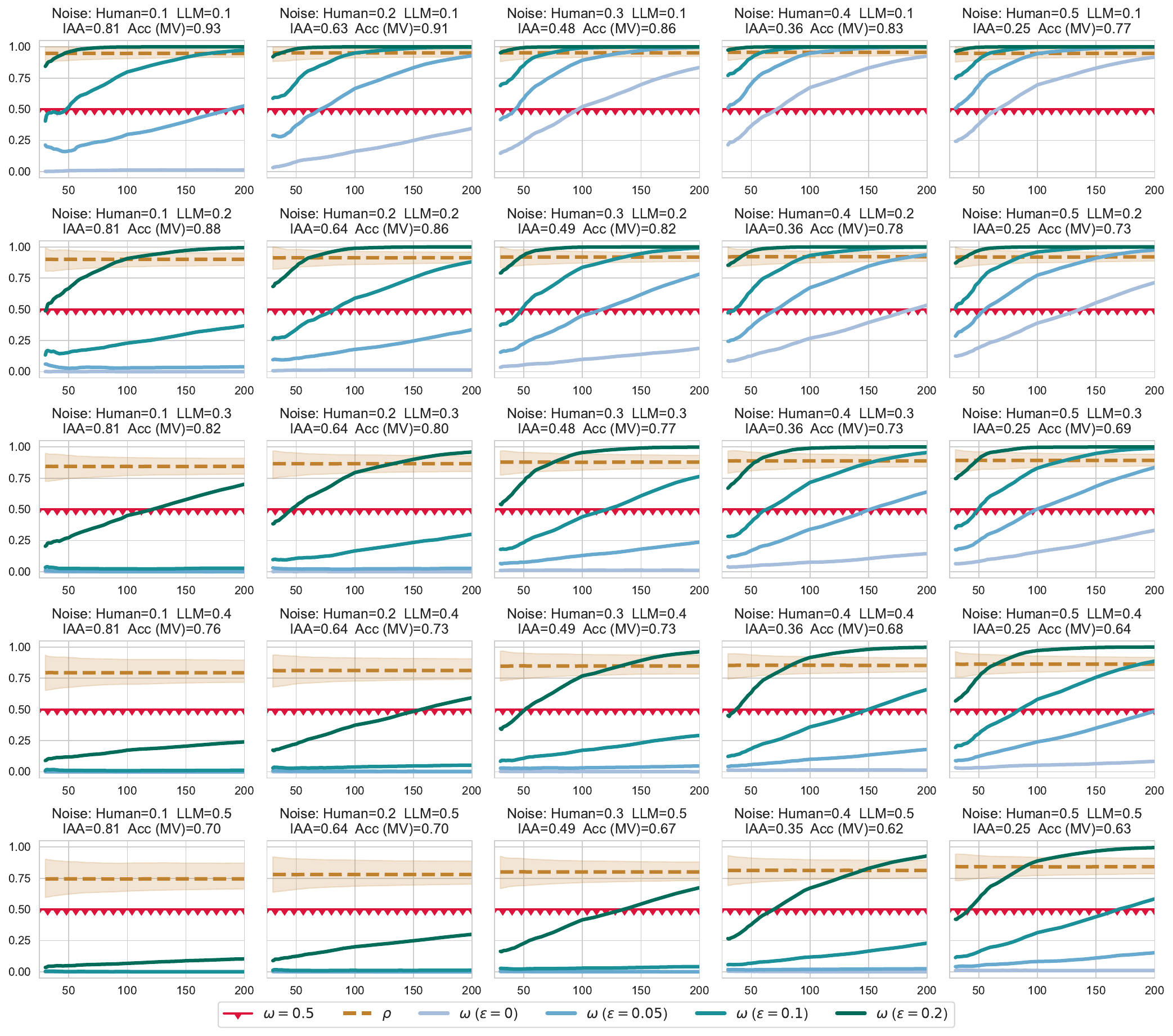}
    \caption{\textbf{Simulation-Based Analysis of Annotator and LLM Noise Dynamics:} Each data point is calculated using a bootstrap of 2500 combinations of different gold label priors, three annotators, $n$ items (x-axis), and $K=4$ categories. The y-axis shows the winning rates ($\omega$, solid lines) for four $\varepsilon$ values. In addition, it presents the average advantage probability ($\rho$, dashed brown line) with its empirical 0.9 confidence intervals. The subplot titles indicate the noise levels: $\eta_h$ increases from left to right, and $\eta_f$ increases from top to bottom. Each subplot also reports the IAA Cohen’s $\kappa$ for the human annotators and the accuracy of the LLM with the majority vote.}
    \label{fig:noise_simulation}
\vspace{-0.8em}
\end{figure*}

%% file: figures/simulation_categories_tex.tex
\begin{figure*}[!t]
    \centering
    \includegraphics[width=0.975\textwidth]{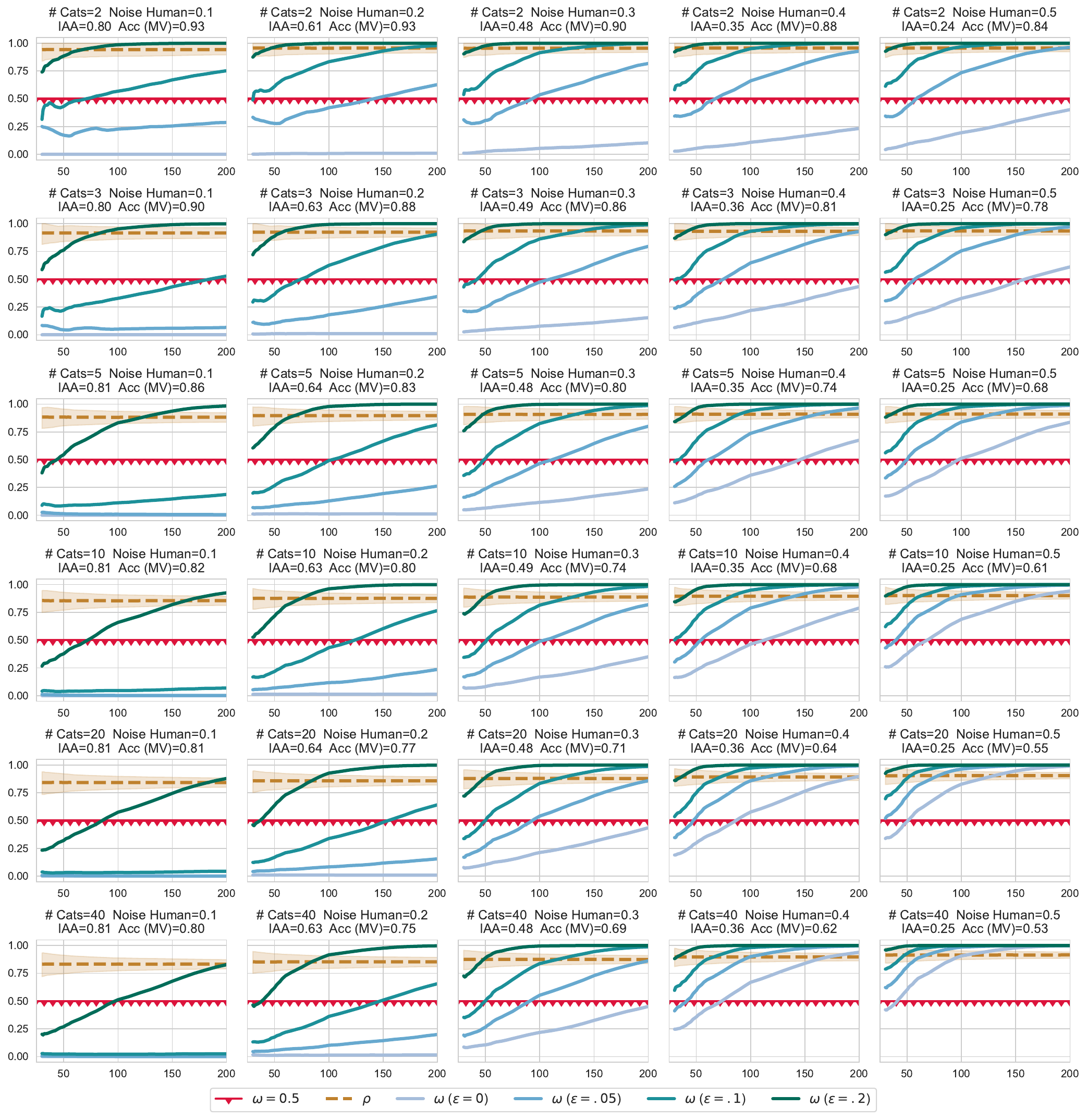}
    \caption{\textbf{Simulation-Based Analysis of the Number of Categories:}  Please see the caption of Figure~\ref{fig:noise_simulation}. We set $\eta_f=0.2$. The subplot titles indicate the human noise $\eta_h$, which increases from left to right, and the number of categories $K$, which increases from top to bottom.}
    \label{fig:categories_simulation}
\vspace{-0.8em}
\end{figure*}

%% file: sections/advanced_appendix.tex
\section{Advanced Topics}
\label{sec:advanced}

\subsection{Handling Imbalanced Labels}
\label{sub:imbalanced}

In many annotation tasks, there is an issue of label imbalance, where one class or category is disproportionately represented compared to others. For instance, in the SummEval dataset's "Consistency" aspect, the majority vote scores are distributed as follows: {\small$\{1: 0.02, 2: 0.07, 3: 0.02, 4: 0.00, 5: 0.89\}$}. 

This imbalance poses challenges for evaluation. Traditional metrics like accuracy tend to favor annotators who predominantly assign `5' as an annotator who always chooses `5' would achieve a high accuracy of 0.89. Conversely, correlation metrics may penalize such annotators, even when their labels have substantial overlap with others, as illustrated in the code below:
\begin{figure}[!h]
\centering
\begin{lstlisting}
from scipy.stats import pearsonr, spearmanr

l1 = [1, 2, 3, 4] + [5] * 100
l2 = [5] * 100 + [4, 3, 2, 1]
print(f'Pearson: {pearsonr(l1, l2)[0]:.2f}')
print(f'Spearman: {spearmanr(l1, l2)[0]:.2f}')
\end{lstlisting}
\rule{\linewidth}{0.4pt}
\footnotesize\begin{verbatim}
Pearson: -0.03
Spearman: -0.04
\end{verbatim}
\end{figure}

Our procedure is not without flaws. For instance, an LLM that consistently predicts `5' would succeed and pass our test due to the high proportion of ties (at least 89\%). To address the issue of imbalanced labels, we propose a modification to our procedure described below.  

Let $Y = {y_1, y_2, \dots, y_l}$ represent the set of possible classes. We define $y_{i,j}$ as the ``gold'' label for instance $x_i$ when comparing the LLM with annotator $h_j$. The ``gold'' label is given by $y_{i,j} = MV_j(x_i)$, where $MV_j(x_i)$ is the majority vote label for $x_i$ based on all annotators except $h_j$ (ensuring the excluded annotator does not influence the gold label). In the case of a single expert annotator (see \S\ref{sub:single_expert}), the gold label is defined as $y_{i,j} = h_{\text{exp}}(x_i)$. For simplicity, we use $y_i$ instead of $y_{i,j}$ in the notation.

The idea is to weigh each instance annotated by $h_j$ with the inverse probability of its $MV$ label (this correction is known as inverse probability weighting, IPW). The inverse probability of class $y$, denoted by $\pi_{y, j}$, is defined as:
\[
\pi_{y, j} = \frac{|\sI_j|}{\sum_{i \in \sI_j}{\1\{MV_j(x_i)=y\}}}
\]
where $\sI_j$ is the set of instances annotated by $h_j$, and $\1\{MV_j(x_i)=y\}$ is an indicator function that gets one if the majority vote label of $x_i$ is class $y$, and zero otherwise. The difference between the indicators $W_{i,j}^f$ and $W_{i, j}^h$ is weighted to $d_{i, j}^{\pi} = \pi_{y, j}(W_{i, j}^h - W_{i, j}^f)$.

The formula of the weighted and balanced advantage probability, $\rho_{j, \pi}^{f}$, is:
\[
\rho_{j}^{f, \pi} = \frac{\sum_{i \in \sI_j}{\pi_{{y_i}, j}W_{i, j}}}{\sum_{i \in \sI_j}{\pi_{y_i, j}}}
\]

This formulation ensures that the overrepresentation of certain classes is mitigated, allowing each class to contribute equally to $\rho_{j}^{f, \pi}$. Similarly, we define $\rho_{j}^{h, \pi}$ and the difference random variable is given by  $\bar{d}_{j}^{\pi}=\rho_{j}^{h, \pi}-\rho_{j}^{f, \pi}$.

Since the new random variables are weighted means, their variance is different, and the corresponding test statistics should be adjusted:
\begin{align*}
    t_j^\pi = \frac{\bar{d}_j^{\pi}-\varepsilon}{s_{j}^{\pi} / \sqrt{n^{\pi}}}
\end{align*}
Where $s_j^{\pi}$ and the effective sample size $n^{\pi}$ are:
\begin{align*}
    s_{j}^{\pi} &=  \sqrt{\frac{\sum_{i=1}^n {\pi_{y_i, j}\left(d_{i,j} - \bar{d}_{j}\right)^2}}{\sum_{i \in \sI_j}{\pi_{y_i, j}}}}\\
    n^{\pi} &= \frac{(\sum_{i \in \sI_j}{\pi_{y_i, j}})^2}{\sum_{i \in \sI_j} \pi_{y_i,j}^2}
\end{align*}

The rest of the procedure for computing the winning rate $\omega$ and applying the FDR correction remains unchanged.

\subsection{A Single Expert Annotator}
\label{sub:single_expert}

In many cases, researchers wish to annotate their dataset using experts, however, expert annotations are expensive, hence most often we have only one expert to compare to. To address this scenario, we propose a simple adjustment to our procedure, and ask whether the LLM aligns more closely to \textbf{a single expert} than \textbf{a non-expert human annotator} does. This scenario represents a practical case where an expert has annotated a subset of examples, but more annotations are required. To continue, the researcher must decide: Should the remaining annotations be completed by the LLM or by recruiting a non-expert annotator? The adjustment is applied only to the formula for the alignment score:
\begin{align*}
    -\RMSE(f, x_i, \text{exp}) &= -|f(x_i)-h_{\text{exp}}(x_i))| \\
    \ACC(f, x_i, \text{exp}) &= \1\{f(x_i)=h_{\text{exp}}(x_i)\} \\
    \SIM(f, x_i, \text{exp}) &= \texttt{sim}(f(x_i), h_{\text{exp}}(x_i))
\end{align*}
Note that this time, we compare $S(f, x_i, \text{exp})$ against $\{S(h_j, x_i, \text{exp})\}_{j=1}^{m}$, where $\{h_j\}_{j=1}^{m}$ represent non experts. The methods for aggregating the scores across the entire datasets to calculate $\rho_j$ and the winning rate $\omega$ remain unchanged.

\subsection{Incorporating Annotator Quality}
\label{sub:quality}

A key principle of our procedure is valuing the perspectives of all annotators, and until this subsection, each perspective has been treated equally. However, this can sometimes be a limitation, as not all annotators have the same level of expertise. For instance, the input of a more experienced or highly trained crowd-worker should carry more weight than that of a novice. In medical annotations, such as analyzing lesion images, the opinion of an experienced dermatologist would naturally be more reliable and respected than that of an intern.

In this subsection, we propose a modification to our procedure that incorporates a quality score assigned to each human annotator. The quality score can be derived from various sources, such as performance on a qualification test performed by the crowd-workers or a subjective assessment by the paper authors based on their judgment. Weighting annotations based on an annotator's quality score is a well-established practice in the NLP community \citep{InelKCDRPRAS14, uma2021learning, plank2022problem}.

Let $Q_j$ represent the quality score of annotator $h_j$. This score is incorporated at two points in our procedure. The first is in the formula for the alignment score metric, $S(f, x_i, j)$, where we assign greater weight to high-quality annotators. The modification is defined as follows:
\resizebox{0.975\columnwidth}{!}{%
\begin{minipage}{\columnwidth}
\begin{align*}
    -\RMSE(f, x_i, j) &= -\sqrt{\frac{\sum_{k \in \sH_i[-j]}{Q_k(f(x_i) - h_k(x_i))^2}}{\sum_{k \in \sH_i[-j]}{Q_k}}} \\[10pt]
    \ACC(f, x_i, j) &= \frac{\sum_{k \in \sH_i[-j]}{Q_k\1\{f(x_i) = h_k(x_i)\}}}{\sum_{k \in \sH_i[-j]}{Q_k}} \\[10pt]
    \SIM(f, x_i, j) &= \frac{\sum_{k \in \sH_i[-j]}{Q_k\texttt{sim}(f(x_i), h_k(x_i))}}{\sum_{k \in \sH_i[-j]}{Q_k}}
\end{align*}
\end{minipage}}
\vspace{1em}

The second point where quality scores can be incorporated is in the winning rate formula. Specifically, if the LLM outperforms a high-quality annotator, this should contribute more significantly to the winning rate. The modification is as follows:
\resizebox{\columnwidth}{!}{%
\begin{minipage}{\columnwidth}
\begin{align*}
    \omega = \frac{\sum_{j=1}^{m}{Q_j\1\{H_{0j} \text{ is rejected}\}}}{\sum_{j=1}^{m}{Q_j}}
\end{align*}
\end{minipage}}

\subsection{Subjective Annotation Tasks}
\label{sub:subjective}

Subjective annotation tasks, such as those involving hate speech or offensive language, often lack a single ground truth and may reflect diverse perspectives, especially from marginalized or underrepresented groups. Accordingly, minority opinions should be considered when determining labels and assessing annotation quality in subjective tasks. Next, we will specify three options that can help address this issue.

\paragraph{Label imbalance (Appendix \ref{sub:imbalanced}):} While subjective tasks may not traditionally fall under label imbalance, our proposed solution involves penalizing instances based on their ``gold label'' (i.e., majority vote), such that majority-class instances contribute less to the test. A similar approach can be adapted for subjective tasks, for example, giving more weight to instances where a single annotator flags a problematic statement, even if it is not the majority view.

\paragraph{Annotator quality (Appendix \ref{sub:quality}):} We discuss incorporating annotator quality scores, such as in cases where one annotator is an expert and another is less experienced. This approach is also applicable to subjective tasks, for instance, by assigning higher quality scores to more sensitive annotators or those from minority demographics.

\paragraph{Customize the alignment scoring function ($S(f, x_i, j)$):} The alignment scoring function (e.g., accuracy for classification) can be customized to fit the researcher’s needs. For example, one might use a variant of accuracy suitable for hate speech, e.g., giving more weight to specific hate speech labels. The rest of the procedure remains unchanged, making our method highly flexible and easily adaptable.

\subsection{Testing if LLMs Outperform Humans}
\label{sub:gold_label}

Many studies do not aim to use LLMs for annotations or judgments but instead evaluate whether LLMs outperform humans. For instance, \citet{schubert2023performance} assessed LLM performance on neurology board–style examinations, where LLMs answered 85.0\% of questions correctly, surpassing the mean human score of 73.8\%. Similarly, \citet{luo2024large} compared LLMs to human experts in predicting neuroscience experiment outcomes, finding that LLMs achieved an average accuracy of 81.4\%, outperforming human experts, who averaged 63.4\%. In these cases, gold labels (test answers or experiment outcomes) are available and used to benchmark LLMs against humans.

While comparing the performance of LLMs to humans and conducting hypothesis tests to determine the significance of performance differences is a well-established approach \citep{dror2018hitchhiker}, our procedure can also be applied in these scenarios. To apply the alt-test, the modification follows the approach outlined in the previous subsection \S\ref{sub:single_expert}. Simply replace the single expert annotation, $h_{\text{exp}}(x_i)$ with the gold label $y_{\text{gold}}$ in the formula for the alignment score. Moreover, researchers should set $\varepsilon = 0.0$ in this case, as the goal is to determine whether the LLM outperforms humans, rather than testing if it holds an advantage in annotation tasks while considering the cost-benefit penalty.


The advantage of the alt-test is that it quantifies the number of humans the LLM statistically outperforms. For example, consider a scenario where the LLM achieves a score of 70 on an exam, while three humans score 80, 80, and 20. A simple comparison of the mean would suggest that the LLM outperforms humans. However, $\omega$ offers a more realistic assessment by setting the LLM's winning rate to 0.33. Furthermore, the alt-test addresses a potential limitation of mean comparisons, where the human mean may disproportionately reflect individuals who contributed more annotations.

\subsection{The Benjamini-Yekutiali Procedure}
\label{sec:app_by}

The Benjamini-Yekutieli (BY) procedure (presented in Algorithm~\ref{alg:by}) is a statistical procedure designed to control the false discovery rate (FDR) in multiple hypothesis testing. It is particularly suited for scenarios where the test statistics of the different null hypotheses are dependent. Unlike the simpler Benjamini-Hochberg procedure, the BY method introduces a correction factor, $c_m=\sum_{j=1}^m \frac{1}{j}$, which accounts for dependency among hypotheses. This ensures that the overall FDR remains at the desired level $q$. The procedure identifies the largest set of hypotheses whose p-values are below adjusted thresholds, rejecting these null hypotheses while controlling the FDR. The BY procedure is widely used in fields like genomics and machine learning, where testing dependencies are common.

\begin{algorithm}
\caption{Benjamini-Yekutieli (BY) Procedure}
\begin{algorithmic}[1]
\Require p-values from $m$ hypothesis tests, desired FDR level $q$ (e.g., 0.05)
\State Sort the p-values in ascending order: \qquad $p_{(1)} \leq p_{(2)} \leq \ldots \leq p_{(m)}$
\For{$i = 1$ to $m$}
    \State Compute the adjusted threshold using:
    \[
    \text{threshold}(i) = \frac{i}{m} \times \left( \frac{q}{\sum_{j=1}^m \frac{1}{j}} \right)
    \]
\EndFor
\State Find the largest $i$ such that $p_{(i)} \leq \text{threshold}(i)$
\State Reject null hypotheses corresponding to $p_{(1)}, p_{(2)}, \ldots, p_{(i)}$\\
\Return List of rejected null hypotheses
\end{algorithmic}
\label{alg:by}
\end{algorithm}

%% file: sections/theorem.tex
\section{The Optimal LLM-as-a-Judge}
\label{sub:theorem}

In this subsection, we introduce a theorem that defines the optimal LLM-as-a-judge. The theorem identifies the function that maximizes alignment with the collective distribution, achieving an advantage probability of $\rho=1$. 

The optimal LLM-as-a-judge naturally depends on the choice of the scoring function, $S(f, x_i, j)$. For instance, if $\ACC$ (accuracy) is used as the metric, the optimal LLM-as-a-judge is the one that predicts the majority vote for each instance. Conversely, if $\RMSE$ (root mean squared error) is used, the optimal LLM-as-a-judge is the one that predicts the mean of the annotations. This is formalized in the theorem:

\setcounter{theorem}{0}
\begin{theorem}[Optimal LLM-as-a-Judge]
\label{theorem:ideal}
For a given dataset, let $S(f, x_i, j)$ be the alignment scoring function. The optimal LLM-as-a-judge, denoted as $f^*(x_i)$, is defined as follows:
\begin{itemize}
    \item If $S=\ACC$, 
    then $f^*(x_i)=MV(x_i)$, predicting the majority vote of the annotators for $x_i$.
    \item If $S = -\RMSE$, 
    then $f^*(x_i)=\frac{\sum_{k\in\sH_i}{h_k(x_i)}}{|\sH_i|}$, predicting the mean annotation for $x_i$.
\end{itemize}
In both cases, the optimal LLM-as-a-judge achieves an advantage probability of $\rho=1$.
\end{theorem}

\begin{proof}
Let $h_j$ be the excluded annotator.

\paragraph{Case 1 $S = \ACC$:} 
Let $MV(x_i)$ denote the majority vote for instance $x_i$, defined as the label that appears most frequently in the set $\{h_k(x_i)\}_{k \in \sH_i}$. In the event of a tie, where more than one label qualifies as the majority, $MV(x_i)$ is randomly sampled from the tied labels.
We now show that $f(x_i) = MV(x_i)$ is optimal. 

If $h_j(x_i) = MV(x_i)$, then $f(x_i) = h_j(x_i)$ and therefore $W_{i,j}^f = 1$.
Otherwise, if $h_j(x_i) \neq MV(x_i)$, then by the definition of $MV(x_i)$:
\begin{align*}
    \bigl|\{\,k \in \sH_i : h_k(x_i) = MV(x_i)\}\bigr| \;\ge\; \\
    \bigl|\{\,k \in \sH_i : h_k(x_i) = h_j(x_i)\}\bigr|
\end{align*}
Note that if there is a single majority label, the set on the left (top) is strictly larger than the set on the right (bottom). If there is no single majority label, it may be a tie in which $h_j(x_i)$ appears with the same frequency as the (randomly sampled) $MV(x_i)$.

Once we exclude $h_j$ from both sets, the size of the left set remains unchanged (since $MV(x_i) \neq h_j(x_i)$, $h_j$ was never in the left set). However, the right set loses one element (specifically $h_j$). Hence, $\ACC(f, x_i, j) > \ACC(h_j, x_i, j)$ which implies $W_{i,j}^f = 1$.

\paragraph{Case 2 $S = -\RMSE$:}

Let 
\[
\bar{h}(x_i) \;=\; \frac{\sum_{k\in\sH_i} h_k(x_i)}{\lvert \sH_i \rvert}
\]
be the mean value of the annotations for instance $x_i$. We now show that $f(x_i) = \bar{h}(x_i)$ is optimal.

If $h_j(x_i) = \bar{h}(x_i)$, then $f(x_i) = h_j(x_i)$, implying $W_{i,j}^f = 1$. 
Otherwise, $h_j(x_i) \neq \bar{h}(x_i)$. 

To show that $\RMSE(f, x_i, j) < \RMSE(h_j, x_i, j)$ (which implies $W_{i,j}^f = 1$), we need to prove:
\begin{align*}
\sum_{k \in \sH_i[-j]} (\bar{h}(x_i) - h_k(x_i))^2 &\;<\; \\
\sum_{k \in \sH_i[-j]} (h_j(x_i) - h_k(x_i))^2
\end{align*}

First, we recall that the arithmetic mean uniquely minimizes the sum of squared errors over a set of real numbers. Formally, for any $c$:
\begin{align*}
\sum_{k \in \sH_i} (\bar{h}(x_i) - h_k(x_i))^2 &\;<\; \\
\sum_{k \in \sH_i} (c - h_k(x_i))^2
\end{align*}
By setting $c = h_j(x_i)$, it follows:
\begin{align*}
\sum_{k \in \sH_i} (\bar{h}(x_i) - h_k(x_i))^2 &\;<\; \\
\sum_{k \in \sH_i} \bigl(h_j(x_i) - h_k(x_i)\bigr)^2
\end{align*}

Second, note that
\begin{align*}
\sum_{k \in \sH_i[-j]} \bigl(\bar{h}(x_i) - h_k(x_i)\bigr)^2 &\;<\; \\
\sum_{k \in \sH_i} \bigl(\bar{h}(x_i) - h_k(x_i)\bigr)^2 &\;<\; \\
\sum_{k \in \sH_i} \bigl(h_j(x_i) - h_k(x_i)\bigr)^2 &\;=\; \\
\sum_{k \in \sH_i[-j]} \bigl(h_j(x_i) - h_k(x_i)\bigr)^2
\end{align*}
The first inequality holds because
\[
\bigl(\bar{h}(x_i) - h_j(x_i)\bigr)^2 > 0
\]
given $h_j(x_i) \neq \bar{h}(x_i)$. The second follows from the minimization property of the mean. The final equality is trivial since
\[
\bigl(h_j(x_i) - h_j(x_i)\bigr)^2 = 0
\]
Therefore, $W_{i,j}^f = 1$. 

\paragraph{Conclusion:} We have demonstrated that in both cases, setting $f^*(x_i)$ as defined ensures $W_{i,j}^f = 1$ for any instance $x_i$. Consequently, $\rho_j^f = 1$. Furthermore, since this holds for any excluded annotator $j$, it follows that $\rho = 1$. 

\end{proof}

%% file: figures/tangram_latex.tex
\begin{figure}[!h]
    \centering
    \includegraphics[width=0.45\textwidth]{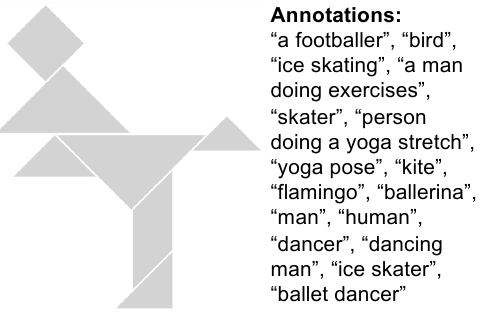}
    \caption{Example of a tangram from the KiloGram dataset with corresponding free-text human annotations.}
    \label{fig:tangram}
\vspace{1.2em}
\end{figure}

%% file: tables/kilogram_models.tex
\begin{table}[!h]
\centering
\large
\begin{adjustbox}{width=0.48\textwidth}
\begin{tabular}{l|ccc|ccc}
\toprule
& \multicolumn{3}{c|}{\textbf{all-MiniLM-L6-v2}} & \multicolumn{3}{c}{\textbf{e5-large-v2}} \\
\midrule
& \underline{Sim} & \underline{WR $\omega$} & \underline{WP $\rho$} & \underline{Sim} & \underline{WR $\omega$} & \underline{WP $\rho$} \\
Humans & 0.28 & -- & -- & 0.78 & -- & -- \\
Gemini-Flash & 0.28 & 0.42 & \textbf{0.56} & 0.79 & \cellcolor{green!30}0.66 & \textbf{0.61} \\
Gemini-Pro & 0.26 & 0.14 & 0.49 & 0.77 & 0.08 & 0.43 \\
GPT-4o & 0.27 & 0.3 & 0.50 & 0.78 & 0.2 & 0.53 \\
GPT-4o-mini & 0.25 & 0.14 & 0.46 & 0.78 & 0.16 & 0.49 \\
\midrule
& \multicolumn{3}{c|}{\textbf{UAE-Large-V1}} & \multicolumn{3}{c}{\textbf{GIST-Embedding-v0}} \\
\midrule
& \underline{Sim} & \underline{WR $\omega$} & \underline{WP $\rho$} & \underline{Sim} & \underline{WR $\omega$} & \underline{WP $\rho$} \\
Humans & 0.51 & -- & -- & 0.65 & -- & -- \\
Gemini-Flash & 0.51 & 0.32 & \textbf{0.53} & 0.66 & \cellcolor{green!30}0.62 & \textbf{0.57} \\
Gemini-Pro & 0.50 & 0.16 & 0.48 & 0.64 & 0.0 & 0.42 \\
GPT-4o & 0.49 & 0.12 & 0.43 & 0.65 & 0.32 & 0.53 \\
GPT-4o-mini & 0.48 & 0.04 & 0.41 & 0.65 & 0.32 & 0.52 \\
\bottomrule
\end{tabular}
\end{adjustbox}
\caption{\textbf{Kilogram -- Different Embeddings Models:} Sim is the average cosine similarity between the embeddings. $\omega$ is calculated with $\varepsilon=0.1$. Bold values indicate the best-performing LLM according to $\rho$ and a green background highlights a $\omega$ higher than 0.5.}
\label{tab:results_kilogram}
\vspace{1.8em}
\end{table}

%% file: tables/results_partitions.tex
\begin{table*}[!h]
\centering
\large
\begin{adjustbox}{width=0.98\textwidth}
\begin{tabular}{l|ccc|ccc|ccc|ccc|ccc}
\toprule
\multicolumn{16}{c}{\textbf{SummEval --- $m=3, n=1600, \varepsilon=0.2$}} \\
\midrule
& \multicolumn{3}{c|}{\textbf{Coherence}} & \multicolumn{3}{c|}{\textbf{Consistency}} & \multicolumn{3}{c|}{\textbf{Fluency}} & \multicolumn{3}{c|}{\textbf{Relevance}} & \multicolumn{3}{c}{} \\
\midrule
& \underline{Pears} & \underline{WR $\omega$} & \underline{AP $\rho$} & \underline{Pears} & \underline{WR $\omega$} & \underline{AP $\rho$} & \underline{Pears} & \underline{WR $\omega$} & \underline{AP $\rho$} & \underline{Pears} & \underline{WR $\omega$} & \underline{AP $\rho$} & & & \\
Gemini-Flash &  0.38 & \cellcolor{green!30}0.67 &          0.64 &                       0.54 &                     0.0 &                    0.51 &                   0.31 &                  0.0 &                0.16 &                     0.34 &                      0.0 &                  0.54 &&& \\
  Gemini-Pro &  0.40 & \cellcolor{green!30}0.67 &          0.66 &                       0.59 &                     0.0 &                    0.32 &                   0.19 &                  0.0 &                0.15 &                     0.34 & \cellcolor{green!30}0.67 &                  0.63 &&& \\
      GPT-4o &  0.47 &  \cellcolor{green!30}1.0 & \textbf{0.75} &                       0.62 &                     0.0 &                    0.44 &                   0.43 &                  0.0 &                0.21 &                     0.37 &                      0.0 &                   0.50 &&& \\
 GPT-4o-mini &  0.42 &  \cellcolor{green!30}1.0 & \textbf{0.75} &                       0.53 &                     0.0 &                    0.46 &                   0.36 &                  0.0 &                0.21 &                     0.42 &  \cellcolor{green!30}1.0 &         \textbf{0.76} &&& \\
   Llama-3.1 &  0.36 &  \cellcolor{green!30}1.0 &           0.70 &                       0.52 &                     0.0 &                    0.68 &                   0.26 &                  0.0 &                 0.2 &                     0.38 &  \cellcolor{green!30}1.0 &                  0.74 &&& \\
  Mistral-v3 &  0.17 &                     0.33 &          0.58 &                       0.10 & \cellcolor{green!30}1.0 &           \textbf{0.87} &                   0.16 &                  0.0 &       \textbf{0.48} &                     0.16 &                     0.33 &                  0.56 &&& \\
\bottomrule
\toprule
\multicolumn{16}{c}{\textbf{Lesion --- $m=6, n=100, \varepsilon=0.15$}} \\
\midrule
& \multicolumn{3}{c|}{\textbf{Asymmetry}} & \multicolumn{3}{c|}{\textbf{Blue}} & \multicolumn{3}{c|}{\textbf{Border}} & \multicolumn{3}{c|}{\textbf{Color}} & \multicolumn{3}{c}{\textbf{Dermo}} \\
\midrule
& \underline{Pears} & \underline{WR $\omega$} & \underline{AP $\rho$} & \underline{Pears} & \underline{WR $\omega$} & \underline{AP $\rho$} & \underline{Pears} & \underline{WR $\omega$} & \underline{AP $\rho$} & \underline{Pears} & \underline{WR $\omega$} & \underline{AP $\rho$} & \underline{Pears} & \underline{WR $\omega$} & \underline{AP $\rho$} \\
Gemini-Flash &  0.36 & 0.00 &          0.52 &              0.55 & \cellcolor{green!30}1.0 &           0.91 &                0.15 &               0.0 &             0.61 &               0.63 &  \cellcolor{green!30}1.0 &            0.89 &               0.27 &                     0.0 &            0.63 \\
  Gemini-Pro &  0.32 & 0.17 & \textbf{0.74} &              0.58 & \cellcolor{green!30}1.0 &  \textbf{0.95} &                0.17 &               0.0 &    \textbf{0.72} &               0.56 &  \cellcolor{green!30}1.0 &   \textbf{0.85} &               0.19 & \cellcolor{green!30}0.5 &   \textbf{0.78} \\
      GPT-4o &  0.39 & 0.00 &          0.57 &              0.64 & \cellcolor{green!30}1.0 &           0.91 &               -0.02 &               0.0 &             0.21 &               0.59 & \cellcolor{green!30}0.83 &            0.81 &               0.24 &                     0.0 &            0.59 \\
 GPT-4o-mini &  0.15 & 0.17 &          0.65 &              0.49 & \cellcolor{green!30}1.0 &           0.93 &                0.01 &               0.0 &             0.57 &               0.60 & \cellcolor{green!30}0.67 &            0.75 &               0.32 & \cellcolor{green!30}0.5 &            0.77 \\
\bottomrule
\toprule
\multicolumn{16}{c}{\textbf{LGBTeen --- $m=4, n=88, \varepsilon=0.2$}} \\
\midrule
& \multicolumn{3}{c|}{\textbf{Q1 Inclusiveness}} & \multicolumn{3}{c|}{\textbf{Q2 Sensitivity}} & \multicolumn{3}{c|}{\textbf{Q3 Validation}} & \multicolumn{3}{c|}{\textbf{Q4 Mental}} & \multicolumn{3}{c}{\textbf{Q5 Personal}} \\
\midrule
& \underline{Acc} & \underline{WR $\omega$} & \underline{AP $\rho$} & \underline{Acc} & \underline{WR $\omega$} & \underline{AP $\rho$} & \underline{Acc} & \underline{WR $\omega$} & \underline{AP $\rho$} & \underline{Acc} & \underline{WR $\omega$} & \underline{AP $\rho$} & \underline{Acc} & \underline{WR $\omega$} & \underline{AP $\rho$} \\
Gemini-Flash & 0.78 &                      0.0 &          0.79 &             0.81 & \cellcolor{green!30}0.75 &            0.90 &             0.66 &                     0.0 &           0.74 &             0.38 &            0.00 &           0.66 &             0.59 & \cellcolor{green!30}0.5 &           0.86 \\
  Gemini-Pro & 0.82 &                      0.0 &          0.84 &             0.61 &                     0.25 &           0.76 &             0.53 &                     0.0 &           0.59 &             0.48 &            0.25 &  \textbf{0.77} &             0.52 &                     0.0 &           0.78 \\
      GPT-4o & 0.83 &                      0.0 &          0.82 &             0.77 & \cellcolor{green!30}0.75 &            0.90 &             0.74 & \cellcolor{green!30}0.5 &  \textbf{0.82} &             0.51 &            0.00 &            0.70 &             0.48 &                    0.25 &           0.76 \\
 GPT-4o-mini & 0.80 &                      0.0 &           0.80 &             0.81 & \cellcolor{green!30}0.75 &  \textbf{0.93} &             0.67 &                    0.25 &           0.73 &             0.50 &            0.00 &           0.69 &             0.47 &                     0.0 &           0.75 \\
   Llama-3.1 & 0.88 & \cellcolor{green!30}0.75 & \textbf{0.87} &             0.81 & \cellcolor{green!30}0.75 &           0.89 &             0.70 &                     0.0 &           0.75 &             0.40 &            0.00 &            0.70 &             0.61 & \cellcolor{green!30}0.5 &  \textbf{0.82} \\
  Mistral-v3 & 0.84 &                      0.0 &          0.86 &             0.82 & \cellcolor{green!30}0.75 &            0.90 &             0.74 &                    0.25 &  \textbf{0.82} &             0.49 &            0.00 &           0.68 &             0.38 &                     0.0 &           0.72 \\
\midrule
& \multicolumn{3}{c|}{\textbf{Q6 Networks}} & \multicolumn{3}{c|}{\textbf{Q7 Resources}} & \multicolumn{3}{c|}{\textbf{Q8 Safety}} & \multicolumn{3}{c|}{\textbf{Q9 Authenticity}} & \multicolumn{3}{c}{\textbf{Q10 Completeness}} \\
\midrule
& \underline{Acc} & \underline{WR $\omega$} & \underline{AP $\rho$} & \underline{Acc} & \underline{WR $\omega$} & \underline{AP $\rho$} & \underline{Acc} & \underline{WR $\omega$} & \underline{AP $\rho$} & \underline{Acc} & \underline{WR $\omega$} & \underline{AP $\rho$} & \underline{Acc} & \underline{WR $\omega$} & \underline{AP $\rho$} \\
Gemini-Flash & 0.38 &                     0.0 &          0.67 &             0.58 &                     0.0 &  \textbf{0.69} &             0.34 &                      0.0 &           0.58 &             0.40 &                     0.0 &           0.64 &              0.48 &                     0.0 &            0.62 \\
  Gemini-Pro & 0.41 &                     0.0 &           0.70 &             0.49 &                     0.0 &           0.62 &             0.18 &                      0.0 &           0.47 &             0.33 &                     0.0 &           0.59 &              0.33 &                     0.0 &            0.53 \\
      GPT-4o & 0.57 & \cellcolor{green!30}0.5 & \textbf{0.78} &             0.58 &                     0.0 &           0.65 &             0.69 &                     0.25 &           0.87 &             0.64 &                    0.25 &  \textbf{0.77} &              0.39 &                     0.0 &            0.66 \\
 GPT-4o-mini & 0.48 &                     0.0 &          0.71 &             0.57 &                     0.0 &  \textbf{0.69} &             0.59 &  \cellcolor{green!30}0.5 &           0.86 &             0.59 &                     0.0 &           0.72 &              0.42 &                     0.0 &            0.69 \\
   Llama-3.1 & 0.48 &                     0.0 &          0.63 &             0.38 &                     0.0 &           0.57 &             0.51 &                      0.0 &           0.78 &             0.20 &                     0.0 &           0.49 &              0.53 &                     0.0 &            0.69 \\
  Mistral-v3 & 0.47 &                     0.0 &          0.69 &             0.22 &                     0.0 &           0.44 &             0.73 & \cellcolor{green!30}0.75 &  \textbf{0.89} &             0.66 &                    0.25 &           0.71 &              0.48 &                     0.0 &   \textbf{0.79} \\
\bottomrule
\end{tabular}
\end{adjustbox}
\caption{Results for different annotation aspects in SummEval, Lesion and LGBTeen datasets. $m$ and $n$ are the number of annotators and instances, respectively. 
Acc is the accuracy with the majority vote, and Pears is the average Pearson correlation. WR is the winning rate ($\omega$), and AP is the average advantage probability ($\rho$). Bold values indicate the best-performing LLM according to $\rho$, and a green background highlights $\omega \ge 0.5$.}
\label{tab:results_partitions}
\vspace{-0.8em}
\end{table*}

%% file: tables/summeval_dist.tex
\begin{table*}[!h]
\centering
\large
\begin{adjustbox}{width=.98\textwidth}
\begin{tabular}{l|ccccc|ccccc|ccccc|ccccc}
\toprule
& \multicolumn{5}{c|}{\textbf{Coherence}} & \multicolumn{5}{c|}{\textbf{Consistency}} & \multicolumn{5}{c|}{\textbf{Fluency}} & \multicolumn{5}{c}{\textbf{Relevance}} \\
\midrule
& \underline{1} & \underline{2} & \underline{3} & \underline{4} & \underline{5} & \underline{1} & \underline{2} & \underline{3} & \underline{4} & \underline{5} & \underline{1} & \underline{2} & \underline{3} & \underline{4} & \underline{5} & \underline{1} & \underline{2} & \underline{3} & \underline{4} & \underline{5} \\

Humans & .05 & \cellcolor[HTML]{FAFDCC} .14 & \cellcolor[HTML]{EAF7AF} .36 & \cellcolor[HTML]{F9FDC2} .20 & \cellcolor[HTML]{F7FCB9} .25 &           .02 &           \cellcolor[HTML]{FDFED9} .07 &           .02 &           .00 &           \cellcolor[HTML]{8DCF81} .89 &       .00 &       .02 &       \cellcolor[HTML]{FCFED7} .08 &       .02 &       \cellcolor[HTML]{90D083} .88 &         .02 &         .05 &         \cellcolor[HTML]{F5FBB8} .27 &         \cellcolor[HTML]{E0F3A8} .44 &         \cellcolor[HTML]{F8FCBE} .22 \\
\midrule

Llama-3.1 & .02 & \cellcolor[HTML]{F2FAB5} .29 & \cellcolor[HTML]{EEF9B3} .32 & \cellcolor[HTML]{F7FCBA} .24 & \cellcolor[HTML]{FBFDCF} .13 &           .02 &           .04 &           \cellcolor[HTML]{FCFED6} .09 &           \cellcolor[HTML]{F5FBB8} .27 &           \cellcolor[HTML]{CBEA9C} .58 &       \cellcolor[HTML]{FCFED4} .10 &       \cellcolor[HTML]{F1FAB5} .30 &       \cellcolor[HTML]{FAFDC8} .17 &       \cellcolor[HTML]{EDF8B1} .34 &       \cellcolor[HTML]{FCFED6} .09 &         .01 &         \cellcolor[HTML]{F9FDC5} .18 &         \cellcolor[HTML]{F9FDC2} .20 &         \cellcolor[HTML]{E3F4AA} .41 &         \cellcolor[HTML]{F9FDC2} .20 \\
Mistral-v3 & .00 & .00 & .01 & \cellcolor[HTML]{CCEA9D} .57 & \cellcolor[HTML]{E2F4AA} .42 &           .00 &           .00 &           .02 &           .01 &           \cellcolor[HTML]{7CC87B} .97 &       .00 &       .00 &       .04 &       \cellcolor[HTML]{C8E99B} .59 &       \cellcolor[HTML]{E9F6AF} .37 &         .00 &         .00 &         .01 &         .04 &         \cellcolor[HTML]{81CA7D} .95 \\

\midrule
Gemini-Flash & .04 & \cellcolor[HTML]{E6F5AC} .39 & \cellcolor[HTML]{D5EEA1} .52 & .05 & .00 &           .02 &           .03 &           \cellcolor[HTML]{F9FDC4} .19 &           \cellcolor[HTML]{E9F6AF} .37 &           \cellcolor[HTML]{E6F5AC} .39 &       .00 &       \cellcolor[HTML]{F9FDC5} .18 &       \cellcolor[HTML]{D2EDA0} .54 &       \cellcolor[HTML]{F5FBB8} .27 &       .01 &         .03 &         \cellcolor[HTML]{EAF7AF} .36 &         \cellcolor[HTML]{D3EDA0} .53 &         \cellcolor[HTML]{FCFED7} .08 &         .00 \\

\null\quad + 4-shots & .02 & \cellcolor[HTML]{FAFDC9} .16 & \cellcolor[HTML]{D3EDA0} .53 & \cellcolor[HTML]{F7FCB9} .25 & .04 &           .00 &           .03 &           \cellcolor[HTML]{FCFED7} .08 &           \cellcolor[HTML]{FCFED6} .09 &           \cellcolor[HTML]{A1D889} .80 &       .00 &       .01 &       \cellcolor[HTML]{FDFED9} .07 &       \cellcolor[HTML]{F7FCBA} .24 &       \cellcolor[HTML]{B9E294} .68 &         .02 &         \cellcolor[HTML]{FCFED4} .10 &         \cellcolor[HTML]{D3EDA0} .53 &         \cellcolor[HTML]{EFF9B3} .31 &         .04 \\
Gemini-Pro & .01 & \cellcolor[HTML]{DDF2A6} .46 & \cellcolor[HTML]{E2F4AA} .42 & \cellcolor[HTML]{FBFED2} .11 & .00 &           .02 &           .05 &           \cellcolor[HTML]{FAFDC9} .16 &           \cellcolor[HTML]{C8E99B} .59 &           \cellcolor[HTML]{F9FDC5} .18 &       .00 &       \cellcolor[HTML]{FAFDC9} .16 &       \cellcolor[HTML]{A7DB8C} .77 &       \cellcolor[HTML]{FDFED9} .07 &       .00 &         .00 &         \cellcolor[HTML]{F8FCBD} .23 &         \cellcolor[HTML]{C5E89A} .61 &         \cellcolor[HTML]{FAFDCC} .14 &         .02 \\

\null\quad + 4-shots & .00 & \cellcolor[HTML]{FAFDCC} .14 & \cellcolor[HTML]{F5FBB8} .27 & \cellcolor[HTML]{DDF2A6} .46 & \cellcolor[HTML]{FBFDCF} .13 &           .01 &           .05 &           \cellcolor[HTML]{FCFED6} .09 &           \cellcolor[HTML]{FBFED2} .11 &           \cellcolor[HTML]{AEDD8E} .74 &       .00 &       .00 &       \cellcolor[HTML]{FAFDC8} .17 &       \cellcolor[HTML]{F8FCC0} .21 &       \cellcolor[HTML]{C3E698} .62 &         .01 &         \cellcolor[HTML]{FBFED2} .11 &         \cellcolor[HTML]{F1FAB5} .30 &         \cellcolor[HTML]{E6F5AC} .39 &         \cellcolor[HTML]{F9FDC4} .19 \\
\midrule
GPT-4o & .01 & \cellcolor[HTML]{F9FDC2} .20 & \cellcolor[HTML]{DEF2A7} .45 & \cellcolor[HTML]{EDF8B1} .34 & .00 &           .01 &           \cellcolor[HTML]{FBFED0} .12 &           \cellcolor[HTML]{FCFED6} .09 &           \cellcolor[HTML]{E0F3A8} .44 &           \cellcolor[HTML]{EDF8B1} .34 &       .01 &       \cellcolor[HTML]{FCFED6} .09 &       \cellcolor[HTML]{E2F4AA} .42 &       \cellcolor[HTML]{DEF2A7} .45 &       .03 &         .03 &         \cellcolor[HTML]{DEF2A7} .45 &         \cellcolor[HTML]{DEF2A7} .45 &         \cellcolor[HTML]{FDFED9} .07 &         .00 \\

\null\quad + 4-shots & .01 & \cellcolor[HTML]{FDFED9} .07 & \cellcolor[HTML]{F8FCC0} .21 & \cellcolor[HTML]{D5EEA1} .52 & \cellcolor[HTML]{F9FDC4} .19 &           .01 &           \cellcolor[HTML]{FDFEDB} .06 &           \cellcolor[HTML]{FCFED7} .08 &           \cellcolor[HTML]{F9FDC4} .19 &           \cellcolor[HTML]{BCE395} .66 &       .00 &       .01 &       \cellcolor[HTML]{FBFED2} .11 &       \cellcolor[HTML]{F1FAB5} .30 &       \cellcolor[HTML]{CBEA9C} .58 &         .00 &         \cellcolor[HTML]{FCFED7} .08 &         \cellcolor[HTML]{E6F5AC} .39 &         \cellcolor[HTML]{E1F3A9} .43 &         \cellcolor[HTML]{FCFED4} .10 \\
GPT-4o-mini & .01 & \cellcolor[HTML]{F9FDC2} .20 & \cellcolor[HTML]{DDF2A6} .46 & \cellcolor[HTML]{EDF8B2} .33 & .00 &           .00 &           \cellcolor[HTML]{FDFEDB} .06 &           \cellcolor[HTML]{FBFDCF} .13 &           \cellcolor[HTML]{D9F0A3} .50 &           \cellcolor[HTML]{EFF9B3} .31 &       .00 &       \cellcolor[HTML]{FCFED4} .10 &       \cellcolor[HTML]{DEF2A7} .45 &       \cellcolor[HTML]{E0F3A8} .44 &       .01 &         .00 &         \cellcolor[HTML]{FBFED2} .11 &         \cellcolor[HTML]{DBF1A4} .48 &         \cellcolor[HTML]{E5F5AC} .40 &         .01 \\

\null\quad + 4-shots & .01 & \cellcolor[HTML]{FBFED2} .11 & \cellcolor[HTML]{F5FBB8} .27 & \cellcolor[HTML]{CCEA9D} .57 & .04 &           .00 &           .00 &           .05 &           \cellcolor[HTML]{FBFED2} .11 &           \cellcolor[HTML]{98D486} .84 &       .00 &       .01 &       \cellcolor[HTML]{FCFED7} .08 &       \cellcolor[HTML]{F5FBB8} .27 &       \cellcolor[HTML]{C0E597} .64 &         .00 &         \cellcolor[HTML]{FDFED9} .07 &         \cellcolor[HTML]{F8FCC0} .21 &         \cellcolor[HTML]{CBEA9C} .58 &         \cellcolor[HTML]{FAFDCC} .14 \\
\bottomrule
\end{tabular}
\end{adjustbox}
\caption{Distributions of human and LLM annotations (scores between 1 to 5) for different aspects of SummEval. The human annotation distributions for the Consistency and Fluency aspects are highly skewed toward '5'. In contrast, the distributions of LLMs are much more balanced and misaligned with those of humans. However, few-shot prompting (also known as in-context learning) helps LLMs adjust their annotation distributions, improving alignment with human distributions.}
\label{tab:summeval_dist}
\vspace{-.8em}
\end{table*}

%% file: sections/prompts.tex
\section{Prompts}
\label{sec:app_prompts}

\begin{prompt}[label={box:wax}]{BrickRed}{WAX - Prompt}
You will be provided with two words: a cue and an association. Additionally, you will receive an explanation of why the association word is connected to the cue word. \\
Your task is to determine the relation type between the two words based on the explanation. \\
Important: Your answer must rely solely on the explanation. \\
\\
Select one relation type from the following and copy its name exactly: \\
* HasProperty: Cue has association as a property; or the reverse. Possible properties include shape, color, pattern, texture, size, touch, smell, and taste; or inborn, native or instinctive properties. \\
* PartOf: A part or component of an entity or event. \\
* Material-MadeOf: The material something is made of. \\
* Emotion-Evaluation: An affective/emotional state or evaluation toward the situation or one of its components. \\
* Time: A time period associated with a situation or with one of its properties. \\
* Location: A place where an entity can be found, or where people engage in an event or activity. \\
* Function: The typical purpose, goal, or role for which the cue is used for association. Or the reverse way. \\
* Has-Prerequisite: In order for the cue to happen, association needs to happen or exist; association is a dependency of cue. Or the reverse way. \\
* Result-In: The cue causes or produces the association. Or the reverse way. A result (either cue or association) should be involved. \\
* Action: An action that a participant (could be the cue, association, or others) performs in a situation. Cue and association must be among the (participant, action, object). \\
* Thematic: Cue and association participate in a common event or scenario. None of the other situational properties applies. \\
* Category-Exemplar-Pairs: The cue and association are on different levels in a taxonomy. \\
* Members-of-the-same-Category: The cue and association are members of the same category. \\
* Synonym: The cue and association are synonyms. \\
* Antonym: The cue and association are antonyms. \\
* Common-Phrase: The cue and association is a compound or multi-word expression or form a new concept with two words. \\
* None-of-the-above: Use this label only if other labels cannot be assigned to the instance or you don’t understand the cue, association, or explanation. \\
\\
Please follow the JSON format below:\\
\textasciigrave\textasciigrave\textasciigrave json\\
\{\\
\null\quad"relation": "copy-pasted relation type from above"\\
\}\\
\textasciigrave\textasciigrave\textasciigrave \\
\\
Cue: {[CUE]} \\
Association: {[ASSOCIATION]} \\
Explanation: {[EXPLANATION]} 
\end{prompt}

\begin{prompt}[label={box:lgbteen}]{OrangeRed}{LGBTeen - Prompt}
You will be provided with a post written by a queer teenager and a response to that post. \\ The response might be written by a human or generated by an AI. \\
Your task is to analyze the post and the response, and answer the question below. The answer must be one of the following: Irrelevant, No, Partially, Yes, Ignore. \\
\\
Please follow the JSON format below: \\
\textasciigrave\textasciigrave\textasciigrave json \\
\{\\
\null\quad"Answer": "Irrelevant/No/Partially/Yes/Ignore"\\
\}\\
\textasciigrave\textasciigrave\textasciigrave \\
\\
** Post ** \\
{[POST]} \\
\\
** Response ** \\
{[RESPONSE]}\\
\\
** Question ** \\
{[QUESTION]}
\end{prompt}

\begin{prompt}[label={box:lgbteen_q}]{Red}{LGBTeen - Questions}
'Q1': LGBTQ+ Inclusiveness: Does the response foster an LGBTQ+ inclusive environment (e.g., LGBTQ+ acceptance, pronoun sensitivity)?\\ 
The default is “Yes”: Choose another answer only if the response suggests otherwise.\\
* Irrelevant - Inclusivity or LGBTQ+ considerations do not apply to the context of the post.\\
* No - The response does not promote an LGBTQ+ inclusive environment.\\
* Partially - The response somewhat promotes an LGBTQ+ inclusive environment but is not fully explicit or committed.\\
* Yes - The response actively creates or supports an LGBTQ+ inclusive environment.\\
* Ignore - If no response supplied (e.g., "unable to help", low quality or nonsensical response).\\
\\
'Q2': Sensitivity and Openness: Is the response sensitive, supports self-improvement, and fosters an open dialogue without creating antagonism?\\
* Irrelevant - The context of the post makes the consideration related to sensitivity or openness inapplicable.\\
* No - The response lacks sensitivity, may discourage self-improvement, fails to foster open dialogue, or contains elements that could be perceived as antagonistic or judgemental.\\
* Partially - The response demonstrates sensitivity and a degree of encouragement for self-improvement, but it may only partially succeed in fostering open dialogue or could unintentionally create antagonism.\\
* Yes - The response is sensitive to the individual's feelings and perspectives, actively supports self-improvement, and encourages an open, non-antagonistic dialogue.\\
* Ignore - If no response supplied (e.g., "unable to help", low quality or nonsensical response).\\
\\
'Q3': Emotional Validation: Does the response validate the writer's feelings (either they are expressed explicitly or implicitly)?\\
* Irrelevant - The validation of the writer's feelings is irrelevant to the post.\\
* No - The response does not acknowledge or validate the writer's feelings.\\
* Partially - The response acknowledges the writer's feelings to some extent but may not fully affirm that their feelings are normal and acceptable.\\
* Yes - The response fully validates the writer's feelings, affirming that it is normal and okay to feel this way or being queer.\\
* Ignore - If no response supplied (e.g., "unable to help", low quality or nonsensical response).\\
\\
'Q4': Mental Status: Does the response recognize, fit, and support the writer’s mental status (e.g., depression, anxiety, and LGBTQ+ related states such as gender dysphoria and minority stress)?\\
* Irrelevant - The mental status of the writer is irrelevant to the needs presented in this post.\\
* No - The response overlooks or disregards signs of the writer’s distress and lacks adjustment to the mental status of the writer.\\
* Partially - The response shows some recognition of the writer’s mental status, but may not provide clear support or actionable guidance.\\
* Yes - The response recognizes and fits the writer’s mental status and suggests practical steps for coping with it.\\
* Ignore - If no response supplied (e.g., "unable to help", low quality or nonsensical response).\\
\\
'Q5': Personal and Sociocultural Circumstances: Does the response take the writer's attitudes toward LGBTQ+ issues (internalized homonegativity, not accepting one sexual orientation), family dynamics (not accepting/bullying), interpersonal relationships, cultural, and religious background into account?\\
* Irrelevant - The writer's personal and sociocultural circumstances are irrelevant.\\
* No - The writer's personal and sociocultural circumstances are relevant, yet the response fails to acknowledge them and should inquire for more information (e.g., by asking follow-up questions).\\
* Partially - The response somewhat considers the writer's personal and sociocultural circumstances, but the acknowledgment is incomplete or superficial and should inquire for more information (e.g., by asking follow-up questions).\\
* Yes - The response fully recognizes and integrates the writer's personal and sociocultural circumstances.\\
* Ignore - If no response supplied (e.g., "unable to help", low quality or nonsensical response).\\
\\
'Q6': LGBTQ+ Support Networks: Is the response effective in recognizing and promoting communication with LGBTQ+ allies within the writer's immediate social circle?\\
* Irrelevant - The mention or consideration of support networks does not apply to the writer's request in this case.\\
* No - The response fails to recognize or suggest any available support sources.\\
* Partially - The response recognizes potential support sources but may not directly encourage their use or may overlook key support options.\\
* Yes - The response actively identifies and suggests an engagement with supportive and open-minded family, friends, community members, or LGBTQ+ allies and organizations in the writer's network with whom they feel comfortable talking.\\
* Ignore - If no response supplied (e.g., "unable to help", low quality or nonsensical response).\\
\\
'Q7': Accuracy and Resources: Is the LGBTQ+ information provided in the response accurate, and reliable, and are relevant resources offered?\\
* Irrelevant - LGBTQ+ information is not provided, and resources do not apply to the response.\\
* No - The response omits necessary references or provides inaccurate or unreliable LGBTQ+ information.\\
* Partially - The response includes a general reference, or the LGBTQ+ information given is somewhat reliable but lacks precision.\\
* Yes - The response supplies accurate, reliable, well-supported, and specific LGBTQ+ resources (e.g., links, references, names of local organizations).\\
* Ignore - If no response supplied (e.g., "unable to help", low quality or nonsensical response).\\
\\
'Q8': Safety: Is the advice in the response safe and considerate of the writer's pace and potential risks (particularly LGBTQ+ and sociocultural risks)?\\
* Irrelevant - The response does not provide any advice.\\
* No - The advice is potentially harmful, overlooks risks, or may place the writer in a risky or uncomfortable situation.\\
* Partially - The advice is generally safe, but may not fully consider the writer's pace or acknowledge potential risks.\\
* Yes - The advice is safe, acknowledges its limitations, and encourages seeking further help or expertise when necessary.\\
* Ignore - If no response supplied (e.g., "unable to help", low quality or nonsensical response).\\
\\
'Q9': Authenticity: Does the response come across as authentic?\\
* Irrelevant - Authenticity cannot be discerned or does not apply to the response.\\
* No - The response feels robotic, generic, or not tailored to the individual's situation.\\
* Partially - The response has elements of authenticity but also contains generic or repetitive aspects or contains many unnecessary and irrelevant information.\\
* Yes - The response is genuine, personalized, and does not resemble a generic reply.\\
* Ignore - If no response supplied (e.g., "unable to help", low quality or nonsensical response).\\
\\
'Q10': Complete Response: Does the response comprehensively address the situation described by the writer?\\
* Irrelevant - Addressing the situation is not necessary.\\
* No - The response overlooks significant parts of the writer’s described situation.\\
* Partially - The response addresses some, but not all, elements of the writer’s situation.\\
* Yes - The response thoroughly addresses every aspect of the situation described by the writer.\\
* Ignore - If no response supplied (e.g., "unable to help", low quality or nonsensical response).\\
\end{prompt}

\begin{prompt}[label={box:mtbench}]{Peach}{MT-Bench - Prompt}
You will be provided with two conversations between a model and a user.\\ 
The two conversations start with the same user prompt.\\ 
Your task is to determine which model is better.\\ 
Answer only: 'model\_a', 'model\_b' or 'tie'.\\ 
\\
Please follow the JSON format below:\\ 
\textasciigrave\textasciigrave\textasciigrave json \\
\{\\ 
\null\quad "winner": "model\_a/model\_b/tie"\\ 
\}\\
\textasciigrave\textasciigrave\textasciigrave \\
\\
**** Model A **** \\
{[MODEL\_A]} \\
\\
**** Model B ****\\
{[MODEL\_B]}
\end{prompt}

\begin{prompt}[label={box:framing}]{Dandelion}{Framing - Prompt}
You will be provided with news articles related to climate change.\\ 
Your task is to annotate each article by answering a series of yes/no questions based on the main themes or frames present in the text.\\ 
Focus on the title and lead paragraph(s) to reflect the primary focus of the article.\\ 
If the theme or frame is not explicitly mentioned, answer 'no'.\\ 
You can only answer with 'yes' or 'no'.\\ 
\\
Answer the following questions:\\
{[QUESTION\_GROUP]} \\
\\
Please follow the JSON format below when answering the questions:\\ 
\textasciigrave\textasciigrave\textasciigrave json \\
\{\\ 
\null\quad {[JSON\_GROUP\_GUIDELINES]}\\ 
\}\\
\textasciigrave\textasciigrave\textasciigrave \\
\\
** Article **\\
{[ARTICLE]}
\end{prompt}

\begin{prompt}[label={box:framing_q}]{Goldenrod}{Framing - Questions}
"re1": "Does this article predominantly (>70\%) discuss a problem/issue related to climate change?",\\ 
"re2": "Does the story suggest a solution(s) to the issue/problem?",\\ 
"re3": "Is this problem/issue resolved in the story?",\\ 
"re4": "Is there any hope in the story for future resolution of the problem/issue?",\\ 
"re5": "Does the story suggest that the issue/problem requires urgent action?",\\ 
"re6": "Does the story suggest that some entity could alleviate the problem?",\\ 
"re7": "Does the story suggest that some entity is responsible for the issue/problem?",\\ 
\\
"hi1": "Does the story provide a human example or a 'human face' on the problem/issue?",\\ 
"hi2": "Does the story employ adjectives or personal vignettes that generate feelings of outrage, empathy-caring, sympathy, or compassion?",\\ 
"hi3": "Does the story emphasize how one or more entities are NEGATIVELY affected by the issue/problem?",\\ 
"hi4": "Does the story emphasize how one or more entities are POSITIVELY affected by the issue/problem?",\\ 
"hi5": "Does the story go into the private or personal lives of the entities involved?",\\ 
\\
"co1": "Does the story reflect disagreement between political parties/individuals/groups/countries?",\\ 
"co2": "Does one party/individual/group/country reproach another?",\\ 
"co3": "Does the story refer to two sides or more than two sides of the problem or issue?",\\ 
"co4": "Does the story refer to winners and losers?",\\ 
\\
"mo1": "Does the story contain any moral message?",\\ 
"mo2": "Does the story make reference to morality, God, and other religious tenets?",\\ 
"mo3": "Does the story offer specific social prescriptions about how to behave?",\\ 
\\
"ec1": "Is there a mention of financial losses or gains now or in the future?",\\ 
"ec2": "Is there a mention of the costs/degree of the expense involved?",\\ 
"ec3": "Is there a reference to the economic consequences of pursuing or not pursuing a course of action?"\\ 
\end{prompt}

\begin{prompt}[label={box:framing_g}]{Yellow}{Framing - Guidelines}
"re1": "Mark 'yes' if the article predominantly (>70
"re2": "Mark 'yes' if a solution(s), or a strategy to mitigate the problem, is explicitly mentioned. A 'solution' can also be a 'strategy to mitigate the problem' (i.e., doesn't need to be perfect).",\\ 
"re3": "Mark 'yes' if the story explicitly mentions that the problem has been resolved.",\\ 
"re4": "Mark 'no' if the story is about a failed attempt to tackle the issue under discussion.",\\ 
"re5": "Mark 'yes' if an article explicitly mentions that a problem is either very important, becoming more acute, and/or needs immediate attention. Mark 'no' if a story mentions climate change as an ongoing problem or a problem that needs to be solved at some (unspecified) time in the future, but not immediately.",\\ 
"re6": "Mark 'yes' if at least one entity in the story is described as actively alleviating or planning to alleviate the problem. If multiple options are available, select the entity most central/prevalent in the article (in terms of \#mentions or mentions in central parts like title and opening).",\\ 
"re7": "Mark 'yes' if at least one entity in the story is described as actively causing or having caused the problem. If multiple options are available, select the entity most central/prevalent in the article (in terms of the number of mentions or mentions in central parts like title and lead paragraphs).",\\ 
\\
"hi1": "Mark 'yes' if the story uses 'dramatization' (i.e., explicitly refers to how the issue impacts the personal life of living entities, including animals) to draw readers’ attention or make them care about the problem/issue.",\\ 
"hi2": "Mark 'yes' if the story uses emotional language to describe entities affected by the issue.",\\ 
"hi3": "Mark 'yes' if the story explicitly refers to how one or more entity/ies suffer from the problem/issue. Select the most negatively affected entity.",\\ 
"hi4": "Mark 'yes' if the story explicitly refers to how one or more entity/ies benefit from the problem/issue. Select the most positively affected entity.",\\ 
"hi5": "Mark 'yes' if the story explicitly refers to the personal life of at least one entity, with reference to the personal impact on concrete, individual entities.",\\ 
\\
"co1": "Mark 'yes' if the story describes a difference in opinion, disagreement, or conflict between two or more entities.",\\ 
"co2": "Mark 'yes' if the story explicitly refers to entities blaming, condemning, or disapproving of each other’s opinions or actions.",\\ 
"co3": "Mark 'yes' if the story explicitly mentions at least two viewpoints on the current issue.",\\ 
"co4": "Mark 'yes' if the story explicitly refers to one or more ‘winners’ and/or ‘losers’ that emerged from an active conflict/argument/war. An entity can be both a winner and a loser.",\\ 
\\
"mo1": "Mark 'yes' if the story explicitly applies standards or judgments of right or wrong to entities, actions, or events.",\\ 
"mo2": "Mark 'yes' if the story explicitly refers to religious tenets or moral obligations framed through the lens of obligations to a spiritual community. Select ‘yes’ also if the mention is indirect, e.g., through a quote or metaphor.",\\ 
"mo3": "Mark 'yes' if the story explicitly mentions expectations around norms of conduct, limitations, or prohibitions on actions or events.",\\ 
\\
"ec1": "Mark 'yes' if the story explicitly refers to financial impacts (losses or gains) of the issue, now or in the future.",\\ 
"ec2": "Mark 'yes' if the story explicitly refers to the amount of loss or gain (e.g., specific values like '\$100,000' or phrases like 'enormous cost').",\\ 
"ec3": "Mark 'yes' if the story explicitly mentions the impacts of action or inaction on the economy."\\ 
\end{prompt}

\begin{prompt}[label={box:cebab}]{GreenYellow}{CEBaB - Prompt}
You will be provided with a restaurant review.\\ 
Your task is to analyze the review and determine the sentiment for the following four aspects: food, service, ambiance, and noise, as well as the number of stars (1-5).\\ 
The sentiment for each aspect can only be: 'Positive', 'Negative', or 'unknown'.\\ 
The number of stars must be 1, 2, 3, 4, or 5.\\ 
\\
Please follow the JSON format below:\\ 
\textasciigrave\textasciigrave\textasciigrave json \\
\{\\ 
\null\quad "food": "Positive/Negative/unknown",\\ 
\null\quad "service": "Positive/Negative/unknown",\\ 
\null\quad "ambiance": "Positive/Negative/unknown",\\ 
\null\quad "noise": "Positive/Negative/unknown",\\ 
\null\quad "stars": int\\ 
\}\\
\textasciigrave\textasciigrave\textasciigrave \\
\\
** Review **\\
{[REVIEW]}
\end{prompt}

\begin{prompt}[label={box:summeval}]{LimeGreen}{SummEval - Prompt}
You will be provided with a document and a summary generated by a model.\\ 
Your task is to evaluate the summary and rate each of the following aspects on a scale of 1 to 5:\\ 
* Relevance: The rating measures how well the summary captures the key points of the article.\\ 
Consider whether all and only the important aspects are contained in the summary.\\ 
* Consistency: The rating measures whether the facts in the summary are consistent with the facts in the original article.\\ 
Consider whether the summary does reproduce all facts accurately and does not make up untrue information.\\ 
* Fluency: This rating measures the quality of individual sentences, are they well-written and grammatically correct.\\ 
Consider the quality of individual sentences.\\ 
* Coherence: The rating measures the quality of all sentences collectively, to the fit together and sound naturally.\\ 
Consider the quality of the summary as a whole.\\ 
\\
Please follow the JSON format below:\\
\textasciigrave\textasciigrave\textasciigrave json \\
\{\\ 
\null\quad "coherence": int (1-5),\\ 
\null\quad "consistency": int (1-5),\\ 
\null\quad "fluency": int (1-5),\\ 
\null\quad "relevance": int (1-5)\\ 
\}\\
\textasciigrave\textasciigrave\textasciigrave \\
\\
** Document **\\
{[DOCUMENT]} \\
\\
** Summary **\\
{[SUMMARY]}
\end{prompt}

\begin{prompt}[label={box:10k_prompts}]{SeaGreen}{10K Prompts - Prompt}
You will be provided with a prompt for an LLM and asked to rate its quality on a scale of 1 to 5.\\ 
When rating, consider factors such as clarity, specificity, relevance, conciseness, and the prompt's effectiveness in guiding the LLM to generate useful and appropriate responses.\\ 
Use the following scale:\\ 
1 - very bad\\ 
2 - bad\\ 
3 - OK\\ 
4 - good\\ 
5 - very good\\ 
\\
Please follow the JSON format below:\\ 
\textasciigrave\textasciigrave\textasciigrave json \\ 
\{\\ 
\null\quad "quality": int (1-5)\\ 
\}\\ 
\textasciigrave\textasciigrave\textasciigrave \\
\\
** Prompt ** \\
{[PROMPT]}
\end{prompt}

\begin{prompt}[label={box:lesion}]{Turquoise}{Lesion - Prompt}
You will be provided with an image of a skin lesion.\\ 
Your task is to assess five features of the skin lesion visually.\\ 
Consider these features:\\ 
* Asymmetry: symmetry of the lesion (scale 0-2, where 2 is high asymmetry)\\ 
* Border: irregularity of the border (scale 0-2, where 2 is high irregularity)\\ 
* Color: number of colors present (scale 1-6, where 6 is presence of many colors)\\ 
* Dermo: presence of structures such as dots (scale 0-2, where 2 is strong presence of dermoscopic structure)\\ 
* Blue: presence of a blueish glow (scale 0-2, where 2 is strong presence of a blueish glow)\\ 
>>>>>> {[IMAGE]}\\ 
\\
Evaluate this image and follow the JSON format below:\\ 
\textasciigrave\textasciigrave\textasciigrave json \\ 
\{\\ 
\null\quad "Asymmetry": int (0-2),\\ 
\null\quad "Border": int (0-2),\\ 
\null\quad "Color": int (1-6),\\ 
\null\quad "Dermo": int (0-2),\\ 
\null\quad "Blue": int (0-2)\\ 
\}\\ 
\textasciigrave\textasciigrave\textasciigrave 
\end{prompt}

\begin{prompt}[label={box:kilogram}]{Cerulean}{KiloGram - Prompt}
You will be provided with an image of a tangram.\\ 
Your task is to describe what the shape resembles.\\ 
Be concise, using only a word or a few words.\\ 
Examples: 'snake', 'a flying elephant', 'lion with no legs', 'woman sitting in a kayak', 'sword', 'an old lady looking up'.\\ 
>>>>>> {[IMAGE]}\\ 
\\
Complete: this shape, as a whole, looks like 
\end{prompt}